%% file: main.tex
\documentclass{article}

\usepackage{arxiv}

\usepackage{natbib}

\usepackage{amsfonts}       
\usepackage{nicefrac}       
\usepackage{microtype}      

\usepackage{amsmath,amsthm}
\usepackage{mathtools,thm-restate}

\newcommand{\abs}[1]{\left\vert #1\right\vert}

\newcommand{\scalarprod}[2]{\langle #1, #2\rangle}
\DeclarePairedDelimiter{\ceil}{\lceil}{\rceil}
\DeclarePairedDelimiter\floor{\lfloor}{\rfloor}
\DeclareMathOperator*{\argmin}{arg\,min}
\newcommand\equaldef{\mathrel{\overset{\makebox[0pt]{\mbox{\normalfont\tiny\sffamily def}}}{=}}}

\newtheorem{claim}{Claim}
\newtheorem{theorem}[claim]{Theorem}

\newtheorem{proposition}[claim]{Proposition}

\newtheorem{definition}[claim]{Definition}
\newtheorem*{notation*}{Notation}
\newtheorem*{remark}{Remark}

\title{PAC-learning gains of Turing machines over circuits and neural networks}
\renewcommand{\shorttitle}{\textit{arXiv} Template}

\date{June, 2022}

\renewcommand{\shorttitle}{PAC-learning gains of Turing machines over circuits and neural networks}

\author{ Brieuc Pinon \\
	ICTEAM/INMA\\
	UCLouvain\\
	Louvain-la-Neuve, Belgium \\
	\texttt{brieuc.pinon@uclouvain.be}
    \And
    Raphaël Jungers \\
    ICTEAM/INMA\\
    UCLouvain\\
    Louvain-la-Neuve, Belgium \\
    \texttt{raphael.jungers@uclouvain.be} \\
    \And
    Jean-Charles Delvenne \\
    ICTEAM/INMA\\
    UCLouvain\\
    Louvain-la-Neuve, Belgium \\
    \texttt{jean-charles.delvenne@uclouvain.be}
}

\begin{document}

\maketitle

\begin{abstract}
    A caveat to many applications of the current Deep Learning approach is the need for large-scale data. One improvement suggested by Kolmogorov Complexity results is to apply the minimum description length principle with computationally universal models. We study the potential gains in sample efficiency that this approach can bring in principle. We use polynomial-time Turing machines to represent computationally universal models and Boolean circuits to represent Artificial Neural Networks (ANNs) acting on finite-precision digits.
    
    Our analysis unravels direct links between our question and Computational Complexity results. We provide lower and upper bounds on the potential gains in sample efficiency between the MDL applied with Turing machines instead of ANNs.  Our bounds depend on the bit-size of the input of the Boolean function to be learned. Furthermore, we highlight close relationships between classical open problems in Circuit Complexity and the tightness of these.
\end{abstract}


\section{Introduction}
    Recent years have seen a renew of interest in the development of methods to make inductive inferences in computationally universal programming languages. This work assesses the gain in sample efficiency that can be attained in principle from these methods in comparison to ANNs under simplifying algorithmic assumptions.
    
    This theoretical investigation is similar to comparisons that have been made between ANNs and other classical Machine Learning algorithms to explain the experimental success of the former.

    We present those two branches of Machine Learning research. The branch that develops methods to make inductive inferences in computationally universal programming languages; and the branch that compares encodings for hypotheses based on the induced minimal description sizes to express functions.
    
    \paragraph{Inductive inference in computationally universal programming languages}
    An important aspect of Deep Learning (DL) research is the development of new architectures. Some of these architectures can efficiently address particular problems, like Convolutional Neural Network (CNN) for Computer Vision, Recurrent Neural Networks (RNN) for natural language processing, and Graph Neural Networks (GNN) that can be adapted to a wide variety of applications, see \cite{battaglia2018relational}.
    
    A common critical point in the development of these architectures is the exploitation of prior knowledge about the task at hand. This exploitation is done by imposing on the model some factorized representation. The structure of the representation enforces the a priori known symmetries in the underlying function to learn. This leads to improvement in the sample efficiency, see \cite{xu-2020-reason} for a paper taking this perspective for GNN with associated PAC-learning results.
    
    For CNN, RNN, and GNN, the factorization's structure is classically fixed a priori with respect to prior information about the symmetries in the function to learn. Alternatively, the structure could be learned from the data.
    
    It is thus natural to invest efforts in the construction of learning algorithms with the flexibility to define and use potential abstract structures found in the data. In this line of work, we mention two approaches: \emph{Differentiable programming} which consists in learning with DL architectures close to Turing-complete systems such as in \cite{graves2014neural,graves2016hybrid,joulin2015inferring,kaiser2015neural,sukhbaatar2015end,kurach2016neural,schlag2018learning,dehghani2018universal,santoro2018relational}; and learning in non-differentiable programs written in universal languages for which combinatorial optimization methods such as genetic programming must be used, see \cite{koza2005genetic}. This latter approach is usually referred to as \emph{Inductive Programming} or \emph{Program synthesis} from examples, see \cite{kitzelmann2009inductive} and \cite{gulwani2017program}.
    
    Our work is a theoretical investigation of the potential sample efficiency gains that could be observed from Inductive Programming or learning with these new expressive DL architectures in comparison to classical Artificial Neural Networks.
    
    \paragraph{Models: the necessary sizes to represent functions}
    This last decade DL algorithms allowed to tackle problems that had been impossible to solve until then.
    These successes opened the question: Why DL is more efficient on these complex problems than other classical Machine Learning algorithms, such as shallow neural networks?
    
    A theoretical answer is that depth in Artificial Neural Networks (ANNs) allows us to efficiently encode some classes of functions. More precisely there exists a sequence of functions for which low-depth ANNs need an exponential number of neurons to approximate it and, in comparison, higher-depth ANNs only need a polynomial number of neurons to achieve the same approximation. Examples of references on the subject are \cite{telgarsky2015representation,liang2016deep,eldan2016power}, and, from the perspective of Boolean circuits, \cite{rossman2015average}.
    
    The expressive power of depth in ANNs to efficiently represent functions in terms of the sizes of the hypotheses has also been studied in comparison with Support Vector Machines by \cite{bengio2007scaling} and with Decision Trees in \cite{bengio2010decision}.
    
    Our work is inspired by these comparisons between models. Similarly, we show a separation in terms of the sizes of the hypotheses that are necessary to fit functions. More precisely, we study the advantage that Turing machines have over circuits and neural networks. Pushing this observation to PAC-learning claims, we study potential gains on the number of samples that are necessary to learn Boolean functions by Turing machines.
    
    \paragraph{Objectives and formal choices}
    Given the motivations for induction in computationally universal programming languages, we investigate the sample efficiency gains that we can hope from this approach relative to classical ANNs.
    
    We use Turing machines to represent models based on Turing-complete systems, such as the new expressive DL architectures for example. This choice of the computational model is not determinant since Turing machines can serve as a proxy to study other Turing-complete systems.
    
    We compare Turing machines with Boolean circuits and classical ANNs. By ANNs we mean the simplest form of DL with non-linear activation functions composed on top of linear transformations and without any repetition of the weights such as in CNN or RNN. To represent ANNs we choose a computationally feasible model, the model is not based on real numbers but is discrete and finite.
    
    The \emph{minimum description length} (MDL) principle consists in choosing the hypothesis with the shortest description length while being consistent with the data. It is a classical formalization of Occam's razor principle. The MDL principle will allow us to translate our choices of models ---which are ways to express the hypotheses--- into learning algorithms. It will give us a general and effective way to compare inductive biases posed by different representations such as ANNs and Turing machines for example.
    
    To assess the performance of the models/learning algorithms, we follow the classical PAC-learning framework. However, we do not explore the computational efficiency question of finding the shortest hypothesis; rather we focus on the sample efficiency. In other words, throughout the paper, we neglect the computational resources needed to find the minimum description length hypothesis consistent with the data; but set our attention on the size of the dataset that is necessary to find a hypothesis that is Probably Approximately Correct.
    
    \paragraph{Outline}
    In Section 2, we provide some background on PAC-learning, interpreters, and the MDL principle.
    
    Then, from it, we introduce in Section 3 the critical metric at the heart of our objectives: the \emph{sample efficiency gains} of a model over another.
    
    In Section 4, we prove bounds on the sample efficiency gains that circuits have over (polynomial-time) Turing machines and conversely. In particular, we show that the sample efficiency gains of polynomial-time Turing machines over circuits are at least linear in the input-size of the function to learn. Whether they are superlinear or not is an open question. We connect this question to different open problems from Computational Complexity.

\section{Background}
    \subsection{Our learning problem}
        We want to learn an unknown function $f$ belonging to a \emph{hypothesis class}
        \begin{equation}
            H^n=\{f:\mathcal{B}^n\rightarrow \mathcal{B}\},
        \end{equation}
        where $\mathcal{B}=\{0,1\}$, and thus $H^n$ is finite.
        
        For some fixed integer $n>0$, a \emph{learning problem} is determined by a boolean function $f\in H^n$ and a probability measure $\mathcal{P}$ on $\mathcal{B}^n$, $\mathcal{P}\in\Delta(\mathcal{B}^n)$.
        
        \begin{definition}
            A \emph{learning problem $m$-sample dataset} is a random variable defined as $\left[x_j,f(x_j)\right]_{j=1}^m$ where the $x_j$ are sampled independently and according to the learning problem probability measure $\mathcal{P}$.
        \end{definition}
        
        To \emph{solve} a learning problem is to find an approximation $\hat{f}$ of $f$ from a realization of the learning problem $m$-sample dataset, for some natural $m$.
        
        We formalize the notion of approximation with the classical accuracy.
        \begin{definition}
            The \emph{accuracy} of a function $\hat{f}\in H^n$ with respect to learning problem $f\in H^n,\,\mathcal{P}\in\Delta(\mathcal{B}^n)$ is
            \begin{equation}
            \mathit{acc}_f^\mathcal{P}(\hat{f}) \equaldef \Pr_{x\sim \mathcal{P}}\left[\hat{f}(x)= f(x)\right].
        \end{equation}
        \end{definition}
        
        \begin{definition}
            A \emph{learning algorithm} is a function that given the realization of a learning problem $m$-samples dataset, for some learning problem $f\in H^n,\mathcal{P}\in\Delta(\mathcal{B}^n)$, outputs a function $\hat{f}\in H^n$, for any $n,m\in\mathbb{N}^+$.
        \end{definition}
        We do not specify a representation for the output's function since, as said in the introduction, our work focuses on the sample efficiency of the learning algorithm and not its computational complexity.
        
        The learning algorithm sample efficiency performance will be assessed by PAC-learning claims.
        \begin{definition}
            For any $\epsilon\in(0,\nicefrac{1}{2}),\,\delta\in(0,1)$, an algorithm $\mathcal{A}$ has an $(\epsilon,\delta)$-PAC-learning performance with an $m$-sample dataset on learning problem $(f,\mathcal{P})$ if for all $m'\geq m$
            \begin{equation}
                \Pr_{x\sim \mathcal{P}^{m'}}\left[\mathit{acc}_f^\mathcal{P}\left(\mathcal{A}(\left[x_j, f(x_j)\right]_{j=1}^{m'})\right)\geq1-\epsilon\right]\geq 1-\delta.
            \end{equation}
        \end{definition}
    
    \subsection{Interpreter}
        We now introduce the concept of interpreters, which will allow us to define the notion of description-length of a hypothesis given a model.
        
        \begin{definition}\label{def:interpreter}
            An \emph{interpreter} $\varphi^T$ is a Turing machine computing a two-arguments partial computable binary-valued function $\varphi:\mathcal{B}^*\times \mathcal{B}^*\rightarrow \mathcal{B}\cup\{\bot\}$.
            Where $\mathcal{B}^*$ is the set of finite length binary strings, $\mathcal{B}^*=\cup_{i=0,1,\ldots}\mathcal{B}^i$, and $\bot$ is the symbol representing non-halting executions. 
        \end{definition}
        
        In the rest of the paper, we will identify the Turing machine implementation with its computed function by dropping the $T$ in $\varphi^T$ with exceptions where the distinction is useful.
        
        In our developments, the first argument will correspond to the code/program/hypothesis and the second argument will correspond to the input of the function to learn.
        
        We make Definition \ref{def:interpreter} concrete by presenting some interpreters, see Section \ref{section:interpreters} in the appendices for complete descriptions:
        \begin{itemize}
            \item \textbf{Universal Turing Machine $\mathcal{U}$:}
            \begin{multline*}
                \mathcal{U}([\text{binary encoding of a two-inputs Turing machine } \mathcal{T},\text{first input}], \text{second input})\\
                = \mathcal{T}(\text{first input},\text{second input}).
            \end{multline*}
            
            Note that we will define Turing machines with, only, binary-valued outputs. The encoding of Turing machines and the Universal Turing machine are formally defined in the appendices, Definitions \ref{def:enc-turing} and \ref{def:univ-turing} respectively.
            
            \item \textbf{Polynomial-time Universal Turing Machine $\mathcal{U}^c$:}
            A Universal Turing machine with a limited computation time in $\mathcal{O}(n^c)$ steps, where $n$ is the size of the input and $c\in\mathbb{N}^+$.
            
            This interpreter will allow us to make claims using hypotheses with reasonable running time.
            
            \item \textbf{Boolean circuit interpreter $\mathcal{C}$:}
            A Boolean circuit is a directed acyclic graph where each node is either an input node or a gate. Each input node is associated with an input value, and gates are associated with unary or binary logical operators (OR, AND, and NOT). There is one node with no child (sink), this node is the output node of the circuit.
            
            The topology of the graph is consistent with the nodes' logical association: input nodes' are not the child of any other node, gates with a binary operator have two parent nodes, and gates with a unary logical operator have a unique parent node.
            
            The output of a circuit on a binary input is the result of the output node after the application of the logical operations associated with the gates. In this computation, the input nodes naturally take their values from the input.
            
            Boolean circuits are encoded as binary strings by first noting the input-size, then the number of nodes in the circuit (the circuit's size), and finally by describing the nodes one by one (associated input or logical operation, and the potential parents).
            
            With this encoding, the description-length of a Boolean circuit of size $S$ is in $O(S\log S)$.
            
            Again the application of this interpreter is
            \begin{equation*}
                \mathcal{C}(\text{binary representation of a Boolean circuit}, \text{input}) = \text{output of the circuit}.
            \end{equation*}
            
            \item \textbf{ANNs interpreter:} We define an ANN as a directed acyclic graph whose nodes correspond to either an input or a floating-point operator. Similarly to Boolean circuits, each input node is associated with one variable of the binary input. The floating-point operators are taken from a predefined arbitrary set. This set can contain binary or unary operators such as $+,-,\times,/,\max(.,.),\exp(.)$; it also contains 0-ary/constant operators: the floating-point values themselves.
            
            Again similarly to Boolean circuits, there is a unique output node that has no children.
            
            The values of the nodes, given an input, are determined by the recursive application of the input nodes' association to input variables or of the corresponding floating-point operators until a value for the output node is obtained.
            
            There is a linear relationship between the description-length of a function with the Boolean circuits' or the ANNs' interpreter, see Proposition \ref{prop:ANN-bool-circuit-descr-length} in the appendices. We use this link to present all the results with the interpreter of Boolean circuits, $\mathcal{C}$, while exactly the same results will hold for ANN.
            
            \item \textbf{Support Vector Machines, Binary Decision Trees, CNN, RNN, GNN :}
            An interpreter can be defined for each of these classes. Note that in some cases the length of the binary representation of some functions can be drastically reduced or increased by using these specialized interpreters.
            
            The existence of these interpreters is noted to point out that the formal background presented applies to more than just Turing machines, Boolean circuits, and ANNs. None of these interpreters will be discussed further in this work.
        \end{itemize}

    \subsection{The Minimum description length (MDL) principle}
        We define how the application of the MDL principle with an interpreter gives a learning algorithm.
        
        \begin{notation*}
            Let $\abs{h}$ be the length of a binary string $h$.
        
            We denote, for any interpeter $\varphi$, any $n$ and any function $f\in H^n$:
            \begin{equation*}
                \abs{f}_\varphi \equaldef \min_{h\in\mathcal{B}^*} \abs{h}\; \text{s.t.}\; \varphi(h,x) = f(x),\; \forall x\in\mathcal{B}^n.
            \end{equation*}
            By convention, if the problem is infeasible the value will be $+\infty$.
        \end{notation*}
        
        \begin{definition}\emph{MDL principle with an interpreter $\varphi$: $\mathit{MDL}^\varphi$.}
            From an interpreter $\varphi$, one can create the following learning algorithm, $\mathit{MDL^\varphi}$.
            
            Input: the realization of a learning problem $m$-samples dataset.
            
            Select the output function, $\hat{f}$, corresponding to a minimal-description-length program consistent with the dataset
            \begin{equation}
            \begin{aligned}
                h^*\in&\argmin_{h\in \mathcal{B}^*} & & \abs{h} \\
                & \text{subject to} & & \varphi(h,x_j)=f(x_j), \; j\in\{1,\ldots m\};\\
                & & & \varphi(h,x)\neq \bot,\; \forall x\in \mathcal{B}^n.
            \end{aligned}
            \end{equation}
            The output function is thus $\hat{f}=\varphi(h^*,.)$.
        \end{definition}
        
        Thus, the choice of an interpreter completely determines the learning algorithm with the MDL principle. The inductive bias is fixed toward low complexity (short to express) hypotheses.
        
        Let us remark that, any computable time-limit on the execution of an interpreter $\varphi$ will ensure that $\mathit{MDL}^\varphi$ is computable. In this paper, we will always assume that such an arbitrary sufficiently large time-limit is imposed, say $2^{2^{2^n}}$ for example.
        
        The following PAC-learning bound will be the main proposition used through the paper to go from descriptions' length constraints to PAC-learning affirmations.
        It is based on a classical argument in the PAC-learning literature to show uniform convergence of the accuracy for finite hypothesis classes, see \cite{blumer1987occam}. The argument is simply adapted to our specific MDL framework.
        
        \begin{restatable}[Description-length PAC-guarantee]{proposition}{paclength}
            \label{prop:pac-length}
            There exists constants $a_1,a_2>0$ such that the following holds.
            
            For any interpreter $\varphi$ and associated learning algorithm $\mathit{MDL}^\varphi$, any $(f,\mathcal{P})$ learning problem and any PAC-learning parameters $\epsilon\in(0,\nicefrac{1}{2}),\,\delta\in(0,1)$, $\mathit{MDL}^\varphi$ has an $(\epsilon,\delta)$-PAC-learning performance with an $m$-sample dataset on the learning problem, where
            \begin{equation}
                m=\frac{a_1}{\epsilon} \left[\log \frac{1}{\delta} + \abs{f}_\varphi + a_2\right].
            \end{equation}
        \end{restatable}
        
        For $\varphi=\mathcal{U}$ or $\varphi=\mathcal{C}$ under mild conditions on the size of the circuits considered, the guarantee of Proposition \ref{prop:pac-length} is tight up to a fixed factor for any configurations $(\epsilon,\delta)$, any input-size $n$ and any sufficiently large maximal description-length $\abs{f}_\varphi$ on some learning problems $(f\in H^n,\mathcal{P}\in \Delta(\mathcal{B}^n))$. This is shown using a VC-dimension based argument in the appendices, see Proposition \ref{prop:tight-pac-length} and its associated section.
    
\section{Sample efficiency gains}
    We propose here a PAC-learning criterion to compare the sample efficiency of two arbitrary MDL-based learning algorithms. This criterion will then be applied to MDL with circuits and MDL with Turing machines.
    
    We want to analyze the largest gap, in terms of necessary samples to get some learning performance, that can exist between two learning algorithms. We quantify the number of samples needed by a learning algorithm in order to solve a learning problem as the following.
    \begin{definition}\label{def:min-m}
        The minimal number of samples needed to get an $(\epsilon,\delta)$-PAC-learning performance for a learning algorithm $\mathit{MDL}^\varphi$ on a learning problem $(f,\mathcal{P})$ is
        \begin{multline}\label{eq:m-min}
            m^{\epsilon,\delta}_\varphi(f,\mathcal{P})\equaldef \min \bigl\{\{m\in\mathbb{N}^+|\,\mathit{MDL}^\varphi \text{ has an } (\epsilon,\delta)\text{-PAC-learning performance}\\\text{with an $m$-sample dataset on learning problem $(f,\mathcal{P})$}  \} \cup \{+\infty\}\bigr\}.
        \end{multline}
    \end{definition}
    
    The PAC-learning guarantee given in Proposition \ref{prop:pac-length} provides an upper-bound on the quantity defined in Equation \ref{eq:m-min}.
    
    Our main goal in this work is to study variations of the sample efficiency according to $n$ the input-size of the underlying function to learn. We will thus parametrize our criterion with this quantity.
    
    Moreover, we impose a practical restriction on the description-length of the function to learn.
    
    \begin{definition}\label{def:G}
        Given two interpreters $\varphi$ and $\psi$, any $\epsilon\in(0,\nicefrac{1}{2}),\,\delta\in(0,1)$, and any $n,d\in\mathbb{N}^+$, the \emph{sample efficiency gain} of $\mathit{MDL}^\varphi$ over $\mathit{MDL}^\psi$ is
        
        \begin{equation}
        \begin{aligned}
            G^d_{\varphi\rightarrow\psi}(\epsilon,\delta,n)\equaldef& \sup_{f\in H^n,\mathcal{P}\in\Delta(\mathcal{B}^n)} & & \frac{m^{\epsilon,\delta}_\psi(f,\mathcal{P})}{m^{\epsilon,\delta}_\varphi(f,\mathcal{P})} \\
            & \text{subject to} & & \abs{f}_\varphi\leq n^d.
        \end{aligned}
        \end{equation}
    \end{definition}
    
    The restriction $\abs{f}_\varphi\leq n^d$ naturally applies on the interpreter for which we try to study a potential sample efficiency advantage, $\varphi$; the restriction ensures from the PAC-learning guarantee presented in Proposition \ref{prop:pac-length} an $(\epsilon,\delta)$-PAC-learning performance with a number of samples polynomial in $n$.
    
    We now define an auxiliary metric for two reasons. First, it is a lower bound on the sample efficiency gains and will be used as such in some theorems' proofs. Second, on some questions we only obtained partial results that are expressed through this metric, instead of sample efficiency gains.
    \begin{definition}\label{def:G-tilde}
        Given two interpreters $\varphi$ and $\psi$, any $\epsilon\in(0,\nicefrac{1}{2}),\,\delta\in(0,1)$, any $n,d\in\mathbb{N}^+$, and $a_1,a_2>0$ fixed in Proposition \ref{prop:pac-length},
        \begin{equation}
        \begin{aligned}
            \Tilde{G}^d_{\varphi\rightarrow\psi}(\epsilon,\delta,n) \equaldef& \sup_{f\in H^n,\mathcal{P}\in\Delta(\mathcal{B}^n)} & & \frac{m^{\epsilon,\delta}_\psi(f,\mathcal{P})}{\frac{a_1}{\epsilon}(\log\frac{1}{\delta}+\abs{f}_\varphi+a_2)}\\
            & \text{subject to} & & \abs{f}_\varphi\leq n^d.
        \end{aligned}
        \end{equation}
    \end{definition}
    
    This value can be interpreted as the best sample efficiency gains that can be obtained from $\mathit{MDL}^\varphi$ over $\mathit{MDL}^\psi$ while being provable from the guarantee given in Proposition \ref{prop:pac-length}.
    
    \begin{remark}
        We defined $G$ ---the sample efficiency gain--- as our metric of interest. However, in the literature ---for example, the literature on depth-separation in ANNs--- a usually assumed criterion of comparison is the necessary sizes of the hypotheses to fit functions, since these can usually be translated into PAC-learning guarantees.
        
        We refer the reader to Section \ref{section:descr-length-criterion} in the appendices for a complete development based on the gains in description-length, $\sup_{f\in H^n}\frac{\abs{f}_\psi}{\abs{f}_\varphi}$ instead of $G$; and, to Section \ref{section:tight-prop-and-vc-dim} for links between constraints on the description-length of an hypothesis class and its VC-dimension for Turing machines and circuits.
    \end{remark}

\section{Main results: comparison of Turing machines and circuits}
    Now that the formal background and the metric of interest are defined, we study the sample efficiency gains by learning with Turing machines or Boolean circuits interpreters under the minimum description length principle. All the results are given for Boolean circuits but also hold for ANNs. In other words, the Boolean circuits' interpreter $\mathcal{C}$ can be substituted for the ANNs' interpreter, given in Definition \ref{def:ANN-int}, in all the theorems that we give.
    
    \subsection{Sample efficiency gains of circuits over Turing machines}
        Before studying the sample efficiency gains of Turing machines over Boolean circuits, we present a partial result in the direction of showing that the sample efficiency gains of circuits over Turing machines are below a constant.
        
        The following theorem shows that on any particular learning problem a PAC-guarantee obtained through Proposition \ref{prop:pac-length} for learning with circuits will imply a similar PAC-guarantee for learning with Turing machines.
        \begin{restatable}{theorem}{turingdom}
            There exists a constant $q\in\mathbb{R}^+$ such that for all $\epsilon\in(0,\nicefrac{1}{2}),\,\delta\in(0,1)$ and $n,d\in\mathbb{N}^+$, 
            \begin{equation}
                \Tilde{G}^d_{\mathcal{C}\rightarrow\mathcal{U}}(\epsilon,\delta,n)\leq q.
            \end{equation}
        \end{restatable}
        
        The rest of the paper analyzes the converse question: can we prove an advantage of Turing-complete systems over circuits for learning in terms of sample efficiency gains?
        
    \subsection{Sample efficiency gains of Turing machines over circuits}
        In the next theorem, we show that we can construct a sequence of learning problems such that learning with Turing-complete languages becomes more and more advantageous in terms of sample efficiency over learning with Boolean circuits. Moreover, the advantage grows exponentially in the input-size of the function to learn.
        \begin{restatable}{theorem}{expgain}\label{thm:exp-gain}
            For all $\epsilon\in(0,\nicefrac{1}{2}),\delta\in(0,1),d\in\mathbb{N}^+$ we have
            \begin{equation}
                G^d_{\mathcal{U}\rightarrow\mathcal{C}}(\epsilon,\delta,n)\in \Omega(2^n/n).
            \end{equation}
        \end{restatable}
        
        The theorem shows that the potential advantage of learning with Turing machines over circuits can quickly become significant.
        
        In the proof of Theorem \ref{thm:exp-gain}, an explicit sequence of learning problems with an advantage for Turing machines is given, the functions to learn are defined by a Turing machine that enumerates on the functions in $H^n$ and the Boolean circuits. The functions to be learned are thus hard to compute.
        
        This can be seen as a practical limitation on the scope of this theorem. We will now use polynomial computational limits on our interpreter. Our formalism will use the interpreter of polynomial-time Turing machines ---$\mathcal{U}^c$--- to this effect.
    
    \subsection{Limits on the sample efficiency gains of polynomial-time Turing machine over circuits}
        In this new computationaly constrained setting, we prove a bound on the sample efficiency gains that can be shown from the length-based PAC-guarantee of Proposition \ref{prop:pac-length}.
        \begin{restatable}{theorem}{circpower}\label{thm:circuit-power}
            For all $\epsilon\in(0,\nicefrac{1}{2}),\,\delta\in(0,1),\, c,d\in\mathbb{N}^+$ we have
            \begin{equation}
                \Tilde{G}^d_{\mathcal{U}^c\rightarrow\mathcal{C}}(\epsilon,\delta,n)\in O(n^c\log^2 n).
            \end{equation}
        \end{restatable}
        
        The proof of Theorem \ref{thm:circuit-power} is based on a result of \cite{pippenger1979relations}, their result states that circuits can compute functions as fast as multi-tape Turing machines. More precisely, we use a refinement by \cite{schnorr1976network} that takes into account the size of the involved Turing machine.
    
    \subsection{Sample efficiency gains of polynomial-time Turing machines over circuits are at least linear in the input-size}
        We show a positive result for learning with polynomial-time Turing machines, the sample efficiency gains grow at least (nearly) linearly in the input-size of the function to learn.
        \begin{restatable}{theorem}{lineargain}\label{thm:linear-gain}
            For all $\epsilon\in(0,\nicefrac{1}{2}),\,\delta\in(0,1),\,1<c,d\in\mathbb{N}^+$ and all $\gamma>0$ we have
            \begin{equation}
                G^d_{\mathcal{U}^c\rightarrow\mathcal{C}}(\epsilon,\delta,n)\in\Omega(n^{1-\gamma}).
            \end{equation}
        \end{restatable}
    
    \subsection{Are sample efficiency gains of polynomial-time Turing machines over circuits superlinear in the input-size?}
        The Theorems \ref{thm:circuit-power} and \ref{thm:linear-gain} open the question of whether the growth of the sample efficiency gains is actually a superlinear polynomial in the input-size of the function to learn. 
        
        As we will show, this question connects with open problems in Computational Complexity.
        
        \subsubsection{If gains are superlinear}
            The first open problem with which we make a connection is the existence of a problem in $\textbf{P}$ for which superlinear sized circuits are necessary.
            
            This problem is of importance in Computational Complexity. There are links between the collapse at the first and second level of the polynomial hierarchy and the computability of languages in $\textbf{NP}$ by polynomial-sized families of Boolean circuits, see \cite{karp1980some}.
            
            However, despite years of efforts, the maximal size-lower-bound known on Boolean circuits for a language in \textbf{NP} is linear, see \cite{iwama2002explicit,arora2009computational,jukna2012boolean}.
            
            The Theorem \ref{thm:if-true} shows that proving that the sample efficiency gains are actually superlinear through the PAC-guarantee offered by Proposition \ref{prop:pac-length} solves this frontier.
            It solves the frontier by proving the existence of a problem in \textbf{P}, and thus \textbf{NP}, with superlinear circuit complexity.
            
            \begin{restatable}{theorem}{ifistrue}\label{thm:if-true}
                If there exists $\epsilon\in(0,\nicefrac{1}{2}),\,\delta\in(0,1),\,c,d\in\mathbb{N}^+$ and some $\gamma>0$ such that
                \begin{equation}
                    \Tilde{G}^d_{\mathcal{U}^c\rightarrow\mathcal{C}}(\epsilon,\delta,n) \notin O(n^{1+\gamma})
                \end{equation}
                then there exists a language in $\textbf{P}$ not computable by any sequence of Boolean circuits whose sizes are in $O(n^{1+\tau})$ for some $\tau>0$.
            \end{restatable}
        
        \subsubsection{If gains are not superlinear}
            On the other hand, if the sample efficiency gains are not superlinear in the input-size of the function to learn for polynomial-time Turing machines over circuits then $\textbf{P}\neq\textbf{NP}$.
            \begin{restatable}{theorem}{ifisfalse}\label{thm:if-false}
                If for all $\epsilon\in(0,\nicefrac{1}{2}),\,\delta\in(0,1),\,c,d\in\mathbb{N}^+$ and all $\gamma>0$
                \begin{equation}
                    G^d_{\mathcal{U}^c\rightarrow\mathcal{C}}(\epsilon,\delta,n)\in O(n^{1+\gamma})
                \end{equation}
                then $\textbf{P}\neq\textbf{NP}$.
            \end{restatable}
            
        \subsubsection{Gains are superlinear under circuit lower bounds}
            The contraposite of the last result yields superlinear gains in the case $\textbf{P}=\textbf{NP}$. However, $\textbf{P}=\textbf{NP}$ is not a common assumption in computer science. We propose an alternative in the following theorem. Let $\textbf{E}$ be the set a language decidable with time-complexity $O(2^{O(n)})$ by Turing machines.
            \begin{restatable}{theorem}{underassumption}\label{thm:underassumption}
                If there exists $f\in\textbf{E}$ and $\iota>0$ such that the Boolean circuits computing $f_n$ are at least of size $2^{\iota n}$; then for any $\gamma>0$ there exists $\epsilon\in(0,\nicefrac{1}{2}),\,\delta\in(0,1),\,c,d\in\mathbb{N}^+$ such that we have
                \begin{equation}
                    G^d_{\mathcal{U}^c\rightarrow\mathcal{C}}(\epsilon,\delta,n)\in \Omega(n^{1+\gamma}).
                \end{equation}
            \end{restatable}
            
            We note that the circuit lower-bounds of the assumption are the same as the ones arising in derandomization research, and the theorem is based on a worst-case to average-case hardness result \cite{arora2009computational}. 

\section{Conclusion}
    In this work, we analyzed the sample efficiency gains of Turing machines over Boolean circuits and classical neural networks under the minimum description length principle in the PAC-learning framework. The Turing machines served as a proxy in the analysis for other Turing-complete systems, such as the recently proposed expressive Deep Learning architectures cited in the introduction, or programs (written in computationally universal languages) that can be learned from Inductive Programming techniques.
    
    We showed that learning with expressive models such as Turing machines can yield sample efficiency gains that are exponential in the input-size of the function to learn. Learning with polynomial-time Turing machines can also yield PAC-learning profits in comparison to learning with circuits. The sample efficiency gains grow linearly in the input-size of the function to learn.
    
    Whether they are superlinear or not is an open problem. One of our main results showed that if these sample efficiency gains are not superlinear in the input-size then $\textbf{P}\neq\textbf{NP}$. Another of our main contributions demonstrated that if it is possible to prove superlinear gains using a classical PAC-learning uniform convergence argument, then there exists a problem in $\textbf{P}$ with superlinear circuit complexity. Additionally, under a circuit lower bound the gains are superlinear in the input-size.
    
    A parallel investigation of the gains in terms of description-length to express Boolean functions is also an output of this research, given in \ref{section:descr-length-criterion}.
    
    This paper also leaves some questions open. Some of our results offer bounds on $\Tilde{G}$, thus providing information on the sample efficiency gains that are provable using the PAC-guarantee of Proposition \ref{prop:pac-length}. Improving these results by using $G$ instead would render them independent of any particular PAC-guarantee.
    
    An extension of the analysis to classically used DL architectures such as Convolutional Neural Networks and Recurrent Neural Networks would also broaden our insight on the potential advantage of using Turing-complete systems as models for learning. We leave this for further work.

\bibliographystyle{unsrtnat}
\bibliography{biblio.bib}

\newpage

\appendix
\section{Main proofs}
    \input{proofs.tex}

\section{Description-length gains of Turing machines over circuits and neural networks}
    \input{description-length-theorems.tex}

\input{appendix}

\end{document}

%% file: proofs.tex
\paclength*
    \begin{proof}
        The algorithm $\mathit{MDL}^\varphi$ never outputs a function with a description-length larger than $\abs{f}_\varphi$. There are at most $I=\sum_{i=0}^{\abs{f}_\varphi} 2^i = 2^{\abs{f}_\varphi+1}-1$ functions in this set.
        
        Let's compute an upper-bound on the probability that there exists a function $\tilde{f}$ of description-length smaller than $\abs{f}_\varphi$ such that $\mathit{acc}_f^\mathcal{P}(\tilde{f})<1-\epsilon$ and which is consistent with a dataset composed of $m\geq\frac{1}{\epsilon}\left[\log\frac{I}{\delta}\right]$ samples.
        
        For any function $\tilde{f}$ with $\mathit{acc}_f^\mathcal{P}(\tilde{f})<1-\epsilon$ the probability to be consistent with the dataset is upper-bounded by $(1-\epsilon)^m$ since each sample is drawn independently according to $\mathcal{P}$.
        
        By the union bound the probability of the existence of one low accuracy function consistent with the data is thus upper-bounded by $I(1-\epsilon)^m$.
        
        Which develops in $I(1-\epsilon)^m\leq Ie^{-\epsilon m}\leq\delta$.
        
        Moreover, there exists constants $a_1,a_2>0$ s.t
        \begin{equation}
            \frac{1}{\epsilon}\left[\log\frac{I}{\delta}\right]\leq \frac{a_1}{\epsilon} \left[\log \frac{1}{\delta} + \abs{f}_\varphi + a_2 \right]
        \end{equation}
        independently of the interpreter, learning problem, and PAC-learning parameters. 
    \end{proof}
    
\turingdom*
    \begin{proof}
        From Definition \ref{def:G-tilde} of $\Tilde{G}$
        \begin{equation}
        \begin{aligned}
            \Tilde{G}^d_{\mathcal{C}\rightarrow\mathcal{U}}(\epsilon,\delta,n) =& \sup_{f\in H^n,\mathcal{P}\in\Delta(\mathcal{B}^n)} & & \frac{m^{\epsilon,\delta}_\mathcal{U}(f,\mathcal{P})}{\frac{a_1}{\epsilon}(\log\frac{1}{\delta}+\abs{f}_\mathcal{C}+a_2)}\\
            & \text{subject to} & & \abs{f}_\mathcal{C}\leq n^d.
        \end{aligned}
        \end{equation}
        
        For any $n$, consider any learning problem $(f\in H^n,\mathcal{P}\in\Delta(\mathcal{B}^n))$.
        
        By the Definition \ref{def:min-m} of $m^{\epsilon,\delta}_\varphi$ and the PAC-guarantee given in Proposition \ref{prop:pac-length}, for any $\epsilon\in(0,\nicefrac{1}{2})$ and $\delta\in(0,1)$, we have
        \begin{equation}
            m^{\epsilon,\delta}_\mathcal{U}(f,\mathcal{P})\leq \frac{a_1}{\epsilon}(\log\frac{1}{\delta}+\abs{f}_\mathcal{U}+a_2).
        \end{equation}
        
        By the main theorem of Kolomogorov complexity, Proposition \ref{prop:KT}, we have $\abs{f}_{\mathcal{U}}\leq \abs{f}_\mathcal{C}+K$ for some constant $K$ independent of $f$. The PAC-learning guarantee for $\mathit{MDL}^\mathcal{C}$ can thus be transformed in a guarantee for $\mathit{MDL}^\mathcal{U}$ since
        \begin{equation}
            \frac{a_1}{\epsilon}(\log\frac{1}{\delta}+\abs{f}_\mathcal{U}+a_2)\leq \frac{a_1}{\epsilon}(\log\frac{1}{\delta} + \abs{f}_\mathcal{C}+K+a_2).
        \end{equation}
        
        Following these inequalities, we get
        \begin{equation}
            \Tilde{G}^d_{\mathcal{C}\rightarrow\mathcal{U}}(\epsilon,\delta,n) \leq \frac{\frac{a_1}{\epsilon}(\log\frac{1}{\delta} + \abs{f}_\mathcal{C}+K+a_2)}{\frac{a_1}{\epsilon}(\log\frac{1}{\delta}+\abs{f}_\mathcal{C}+a_2)} \leq \frac{K}{a_2}+1.
        \end{equation}
        
    \end{proof}

\expgain*
    \begin{proof}
        By Definition \ref{def:G} of the sample efficiency gains, we have
        \begin{equation}
        \begin{aligned}
            G^d_{\mathcal{U}\rightarrow\mathcal{C}}(\epsilon,\delta,n)=& \sup_{f\in H^n,\mathcal{P}\in\Delta(\mathcal{B}^n)} & & \frac{m^{\epsilon,\delta}_\mathcal{C}(f,\mathcal{P})}{m^{\epsilon,\delta}_\mathcal{U}(f,\mathcal{P})} \\
            & \text{subject to} & & \abs{f}_\mathcal{U}\leq n^d.
        \end{aligned}
        \end{equation}
    
        For any $\epsilon\in(0,\nicefrac{1}{2}),\delta\in(0,1),d\in\mathbb{N}^+$, we define a sequence of learning problems which prove the statement.
        For any $n$, we define a learning problem to solve. For any $n$, the learning problem is to learn under the uniform distribution, $U$, the binary function computed by the interpretation of the following program.
        
        \begin{itshape}
        For input $x$ of size $n$:
        \begin{enumerate}
            \item Compute the input size $n$.
            \item Enumerate all the functions in $H^n$ in some fixed lexicographic order. For each of these functions, $f$:
                \begin{enumerate}
                    \item Compute the hypothesis of minimal-description-length according to the interpreter $\mathcal{C}$ to represent a function $\hat{f}$ such that $\mathit{acc}^U_f(\hat{f})\geq 1-\epsilon$.
                    
                    Rewritten
                    \begin{equation}
                    \begin{aligned}
                        &\min_{h\in \mathcal{B}^*} & & \abs{h} \\
                        & \text{subject to} & & \mathit{acc}^U_f(\mathcal{C}(h,.))\geq 1-\epsilon.
                    \end{aligned}
                    \end{equation}
                    
                    \item With $H(.)$ the binary entropy function, see Definition \ref{def:bin-ent}, if the optimal description-length is bigger than $2^n(1-H(\epsilon))-2$ then return $f(x)$.
                \end{enumerate}
        \end{enumerate}
        \end{itshape}
        We first show that this is a well-defined computable function in the sense that the step \textit{(b)} will always be satisfied for some function in the enumeration for all $n$ sufficiently large.
        
        The next development shows by a counting argument, that all the binary functions cannot be approximated within $1-\epsilon$ by Boolean circuits of binary description-length smaller than $2^n(1-H(\epsilon))-2$.  
        
        The number of functions that can approximated with accuracy higher than $1-\epsilon$ by Boolean circuits of description-length smaller than $2^n(1-H(\epsilon))-2$ is
        \begin{equation}
            \abs{\bigcup_{l\in 0,\ldots,\floor{2^n(1-H(\epsilon))-2}} \bigcup_{h\in\mathcal{B}^l} \left\{f\in H^n|\,\mathit{acc}_f^U(\mathcal{C}(h,.))\geq 1-\epsilon\right\} }.
        \end{equation}
        
        Using Proposition \ref{prop:approx_count} and the fact that there are $2^{2^n}$ functions in $H^n$, 
        \begin{multline*}
            \abs{\bigcup_{l\in 0,\ldots,\floor{2^n(1-H(\epsilon))-2}} \bigcup_{h\in\mathcal{B}^l} \left\{f\in H^n|\,\mathit{acc}_f^U(\mathcal{C}(h,.))\geq 1-\epsilon\right\} }\\
            \leq \sum_{l\in 0,\ldots,\floor{2^n(1-H(\epsilon))-2}} \sum_{h\in\mathcal{B}^l} \abs{\left\{f\in H^n|\,\mathit{acc}_f^U(\mathcal{C}(h,.))\geq 1-\epsilon\right\}}\\
            = \sum_{l\in 0,\ldots,\floor{2^n(1-H(\epsilon))-2}} \sum_{h\in\mathcal{B}^l} \sum_{i\in 0,\ldots, \floor{\epsilon 2^n}} \binom{2^n}{i}\\
            \leq \sum_{l\in 0,\ldots,\floor{2^n(1-H(\epsilon))-2}} \sum_{h\in\mathcal{B}^l} 2^{2^n H(\epsilon)}\\
            \leq 2^{2^n(1-H(\epsilon))-1} 2^{2^n H(\epsilon)}\\
            =2^{2^n-1}<2^{2^n}=\abs{H^n}.
        \end{multline*}
        
        Thus the condition in step \textit{(b)} will always be satisfied for some function.
        
        Also, the condition $\abs{f}_\mathcal{U}\leq n^d$ will always be satisfied for all $n$ large enough since the learning problem's function corresponds to a program/Turing-machine of fixed description-length.
        
        Moreover, using the guarantee of Proposition \ref{prop:pac-length}, $m^{\epsilon,\delta}_\mathcal{U}(f,\mathcal{P})\leq \frac{a_1}{\epsilon}(\log\frac{1}{\delta}+\abs{f}_\mathcal{U}+a_2)$, we deduce that $m^{\epsilon,\delta}_\mathcal{U}(f,\mathcal{P})$ is upper-bounded by a constant independent of $n$.
        
        Also by construction, $\mathit{MDL}^\mathcal{C}$ has to select circuits of description-length growing at least as fast as $2^n(1-H(\epsilon))-2$ to be able to approximate the function to learn with an average error at most $\epsilon$. We show that this condition has implications on the size of the minimal circuit that has to be selected and then on the minimal number of samples needed.
        
        By Definition \ref{def:bool-circ}, for all $n$ sufficiently large, all circuits of size lower than $\alpha2^n/n$, for any $\alpha>0$, can be described with the Boolean circuit interpreter $\mathcal{C}$ such that their description-length is lower than
        \begin{equation}
            9\alpha\frac{2^n}{n}\log( \alpha 2^n/n).
        \end{equation}
        
        Thus, for all $n$ sufficiently large, the description-lengths of these circuits are lower than
        \begin{equation}
            9\alpha\left[ \log(2)2^n + \log(\alpha)\frac{2^n}{n} - \frac{2^n}{n}\log n \right]\leq
            9\alpha\log(2\alpha) 2^n.
        \end{equation}
        
        Thus there exists a $\alpha$ sufficiently small such that, for all $n$ sufficiently large, the description-lengths of these circuits are smaller than $2^n(1-H(\epsilon))-2$.
        
        Consequently, all these circuits with size at most $\alpha2^n/n$ must be eliminated by $\mathit{MDL}^\mathcal{C}$. This requires at least $\frac{C}{b}2^n/n$ samples, for some fixed $b>0$, by Proposition \ref{prop:bc-expr}.
    \end{proof}

\circpower*

    \begin{proposition}{\cite{pippenger1979relations,schnorr1976network}.}\label{prop:bc-tm}
        If a multi-tape Turing machine $M$ computes a function on inputs of size $n$ within $t$ steps then there exists a Boolean circuit of size at most $\alpha (\text{number of rules of }M) t\log t$ that computes the same function, where $\alpha$ depends only on the number of tapes and the alphabet size of the Turing machine.
    \end{proposition}
        
    \begin{proof}[Proof of Theorem \ref{thm:circuit-power}]
        
        By Definition \ref{def:G-tilde} of $\tilde{G}$
        \begin{equation}\label{eq:G-tilde-circ-power}
        \begin{aligned}
            \Tilde{G}^d_{\mathcal{U}^c\rightarrow\mathcal{C}}(\epsilon,\delta,n) =& \sup_{f\in H^n,\mathcal{P}\in\Delta(\mathcal{B}^n)} & & \frac{m^{\epsilon,\delta}_\mathcal{C}(f,\mathcal{P})}{\frac{a_1}{\epsilon}(\log\frac{1}{\delta}+\abs{f}_{\mathcal{U}^c}+a_2)}\\
            & \text{subject to} & & \abs{f}_{\mathcal{U}^c}\leq n^d.
        \end{aligned}
        \end{equation}
        
        We will upper-bound the ratio for any learning problem $(f\in H^n,\mathcal{P}\in\Delta(\mathcal{B}^n))$. Note that only functions $f$ corresponding to finite $\abs{f}_{\mathcal{U}^c}$ have to be considered.
        
        We use Proposition \ref{prop:hardwire} on the Turing machine of the interpreter $\mathcal{U}^c$ and the input $h$ of length $\abs{f}_{\mathcal{U}^c}$ such that $\mathcal{U}^c(h,.)=f$. From the application of the proposition, we deduce that there exists a Turing machine with at most $\rho \abs{f}_{\mathcal{U}^c}$ rules that compute $f$, where $\rho>0$ is a parameter independent of $f$ and $n$. Moreover, the proposition tells us that the resulting Turing machine has the same computational time-limit as $\mathcal{U}^c$, and is thus also bounded by $\beta n^c$, for some $\beta>0$.
        
        Applying now Proposition \ref{prop:bc-tm} on these facts, we deduce that there exists a Boolean circuit computing $f$ of size at most, with $t=\beta n^c$,
        \begin{equation}
            \alpha\rho\abs{f}_{\mathcal{U}^c}t\log t,
        \end{equation}
        where $\alpha>0$ is independent of $f$ and $n$.
        
        By Definition \ref{def:bool-circ}, the description-length of this circuit with the Boolean circuit interpreter $\mathcal{C}$, and thus $\abs{f}_\mathcal{C}$, will be upper bounded by
        \begin{equation}\label{eq:upper-bound-circ-length}
            2\ceil{\log_2(n)}+2+\underbrace{\alpha\rho\abs{f}_{\mathcal{U}^c}t\log t(3+\max\{2\ceil{\log_2\alpha\rho\abs{f}_{\mathcal{U}^c}t\log t}),\ceil{\log_2 n}\}}_{\equaldef B}.
        \end{equation}
        
        We use the PAC-guarantee of Proposition \ref{prop:pac-length} to upper-bound $m^{\epsilon,\delta}_\mathcal{C}(f,\mathcal{P})$ with the upper-bound on $\abs{f}_\mathcal{C}$ of Equation \ref{eq:upper-bound-circ-length}.
        
        Then the quantity of Equation \ref{eq:G-tilde-circ-power} is also upper-bounded
        \begin{equation}\label{eq:circ-power-upper-bound-G-tilde}
            \begin{split}
                \Tilde{G}_{\mathcal{U}^c\rightarrow\mathcal{C}}(\epsilon,\delta,n) &\leq \sup_{f\in H^n} \frac{\frac{a_1}{\epsilon}(\log\frac{1}{\delta}+2\ceil{\log_2(n)}+2+B+a_2)}{\frac{a_1}{\epsilon}(\log\frac{1}{\delta}+\abs{f}_{\mathcal{U}^c}+a_2)}\;\text{s.t.}\;\abs{f}_{\mathcal{U}^c}\leq n^d,\\
                &\leq \sup_{f\in H^n} \frac{B}{\abs{f}_{\mathcal{U}^c}}+ \frac{2}{a_2}\ceil{\log_2(n)}+\frac{2}{a_2} +1\;\text{s.t.}\;\abs{f}_{\mathcal{U}^c}\leq n^d.
            \end{split}
        \end{equation}
        
        Where $\frac{B}{\abs{f}_{\mathcal{U}^c}}$ equals
        \begin{equation}\label{eq:circ-power-upper-B-on-length}
            \alpha\rho t\log t(3+\max\{2\ceil{\log_2\alpha\rho\abs{f}_{\mathcal{U}^c}t\log t}),\ceil{\log_2 n}\}.
        \end{equation}
        Using $\abs{f}_{\mathcal{U}^c}\leq n^d$ and $t=\beta n^c$, the value in Equation \ref{eq:circ-power-upper-B-on-length} is in
        \begin{equation}
            O(n^c\log^2n).
        \end{equation}
        
        Returning this result to the inequalities of Equation \ref{eq:circ-power-upper-bound-G-tilde}, we obtain
        \begin{equation}
            \Tilde{G}_{\mathcal{U}^c\rightarrow\mathcal{C}}(\epsilon,\delta,n) \in O(n^c\log^2 n).
        \end{equation}
    \end{proof}
    We note that the Theorem \ref{thm:circuit-power} can be improved upon by using another definition for the polynomial-time universal Turing machine, $\mathcal{U}^c$. The Definition \ref{def:l-res-turing} use the construction of \cite{hennie1966two}. Another possibility is to use the construction of \cite{pippenger1979relations} related to Proposition \ref{prop:bc-tm}. This construction would make $\mathcal{U}^c$ oblivious and directly transformable into a Boolean circuit without the additional $\log n$ factor. The final bound would thus be in $O(n^c\log n)$ instead of $O(n^c\log^2 n)$.
    
\lineargain*
    \begin{proof}
        By Definition \ref{def:G}, the sample efficiency gain is
        \begin{equation}
        \begin{aligned}
            G^d_{\mathcal{U}^c\rightarrow\mathcal{C}}(\epsilon,\delta,n)=& \sup_{f\in H^n,\mathcal{P}\in\Delta(\mathcal{B}^n)} & & \frac{m^{\epsilon,\delta}_\mathcal{C}(f,\mathcal{P})}{m^{\epsilon,\delta}_{\mathcal{U}^c}(f,\mathcal{P})} \\
            & \text{subject to} & & \abs{f}_{\mathcal{U}^c}\leq n^d.
        \end{aligned}
        \end{equation}
        
        Fix any combination of $\epsilon\in(0,\nicefrac{1}{2}),\delta\in(0,1),1<c,d\in\mathbb{N}^+$.
    
        For any input size $n$, we contruct the following learning problem, $(f\in H^n,\mathcal{P}\in\Delta(\mathcal{B}^n))$. The function to learn $f$ is the parity function on $\mathcal{B}^n$, noted $\oplus$ (it is the number of $1$ in the input modulo $2$). The probability measure on the domain $\mathcal{B}^n$, $\mathcal{P}$, is the uniform distribution $U^n$. Let $\oplus|_n$ denote the function $\oplus$ restricted to size $n$ inputs.
        
        There exists a fixed Turing machine computing the parity function in linear time for all $n$. By Definition \ref{def:l-res-turing}, for $c>1$ there exists an $h\in\mathcal{B}^*$ such that $\mathcal{U}^c(h,.)$ computes $\oplus$ for all $n$ large enough. Thus $\abs{\oplus|_n}_{\mathcal{U}^c}$ is at most some constant.
        
        A first consequence is that, for $n$ large enough, $\abs{\oplus|_n}_{\mathcal{U}^c}\leq n^d$ will be satisfied for the form of learning problem we defined.
        
        A second consequence, using the PAC-guarantee given in Proposition \ref{prop:pac-length}, is that the denominator $m^{\epsilon,\delta}_{\mathcal{U}^c}(f,\mathcal{P})\leq\frac{a_1}{\epsilon}(\log\frac{1}{\delta}+\abs{\oplus|_n}_{\mathcal{U}^c}+a_2)$ is less than a constant.
        
        Now for the numerator, for any $\gamma>0$ take any sampling of size less than $n^{1-\gamma}$. By Proposition \ref{prop:bc-expr} and Definition \ref{def:bool-circ}, we know that a Boolean circuit of size less than $bn^{1-\gamma}$ will be selected by $\mathit{MDL}^\mathcal{C}$ for some fixed $b>0$.
        
        For $n$ sufficiently large $bn^{1-\gamma}<n$, and thus the circuit selected by $\mathit{MDL}^\mathcal{C}$ will not depend on all the inputs' variables. Suppose without loss of generality that $x_n$ is one of these variables to which the circuit is not sensible. The accuracy of a function $C$ computed by such a selected circuit on our learning problem is
        \begin{equation}
            \begin{split}
                \mathit{acc}_{\oplus}^{U^n}(C) &= \frac{1}{2^n}\sum_{x\in\mathcal{B}^n} \oplus(x) = C(x_1,\ldots,x_{n-1},x_n)\\
                &= \frac{1}{2^n}\sum_{x^-\in\mathcal{B}^{n-1}}\sum_{x_n\in\mathcal{B}} \oplus(x^-,x_n) = C(x^-_1,\ldots,x^-_{n-1})\\
                &= \frac{1}{2^n} \sum_{x^-\in\mathcal{B}^{n-1}} 1\\
                &= \frac{1}{2}.
            \end{split}
        \end{equation}
        
        Since $\epsilon<\nicefrac{1}{2}$, it is impossible for $\mathit{MDL}^\mathcal{C}$ to get a sufficient accuracy with only $n^{1-\gamma}$ samples for $n$ sufficiently large, and so $m^{\epsilon,\delta}_\mathcal{C}(f,\mathcal{P})\in\Omega(n^{1-\gamma})$.
    \end{proof}

\ifistrue*
    \begin{proof}
        Let's suppose superlinear gains $\Tilde{G}^d_{\mathcal{U}^c\rightarrow \mathcal{C}}(\epsilon,\delta,n)\notin O(n^{1+\gamma})$ for some $\epsilon\in(0,\nicefrac{1}{2}),\,\delta\in(0,1),\, c,d\in\mathbb{N}^+$ and $\gamma>0$.
        
        Recall that by Definition \ref{def:G-tilde}, for some constants $a_1,\,a_2>0$ defined in Proposition \ref{prop:pac-length},
        \begin{equation}\label{eq:def-G-tilde-superlinear-true}
        \begin{aligned}
            \Tilde{G}^d_{\mathcal{U}^c\rightarrow\mathcal{C}}(\epsilon,\delta,n) =& \sup_{f\in H^n,\mathcal{P}\in\Delta(\mathcal{B}^n)} & & \frac{m^{\epsilon,\delta}_\mathcal{C}(f,\mathcal{P})}{\frac{a_1}{\epsilon}(\log\frac{1}{\delta}+\abs{f}_{\mathcal{U}^c}+a_2)}\\
            & \text{subject to} & & \abs{f}_{\mathcal{U}^c}\leq n^d.
        \end{aligned}
        \end{equation}
        
        We give the proof outline:
        \begin{enumerate}
            \item First, we prove that the premise of the theorem's statement implies the existence of an infinite sub-sequence of functions with description-length gains superlinear in the input-size. This result appears in Equation \ref{eq:uncomp-gain}.
            \item Second, we present a fixed program of polynomial computational complexity.
            \item Third, we show lower-bounds on some input-sizes for this program with the interpreter for circuits. These lower-bounds come from the developments in the first point.
            \item Fourth, we show from these lower-bounds that the description-length of the program with the circuit interpreter is at least superlinear in the input-size.
            \item Fifth, we conclude by showing superlinear circuit complexity for the presented program. 
        \end{enumerate}
        
        For any $n$, the sets $H_n$ are finite and thus, for any $n$, the supremum can be attained for some $f_n\in H_n$. Then there exists a sequence of functions $(f_n)=f_1\in H_1,\ldots,f_n\in H_n,\ldots$ such that the sequence in $n$
        \begin{equation}
            \sup_{\mathcal{P}\in\Delta(\mathcal{B}^n)} \frac{m^{\epsilon,\delta}_\mathcal{C}(f_n,\mathcal{P})}{\frac{a_1}{\epsilon}(\log\frac{1}{\delta}+\abs{f_n}_{\mathcal{U}^c}+a_2)}
        \end{equation}
        is not in $O(n^{1+\gamma})$, and, morevover, with $f_n\leq n^d$ for all $n$.
        
        By contraposition of the PAC-guarantee offered by Proposition \ref{prop:pac-length}, for all $n$ and $\mathcal{P}$, $\frac{a_1}{\epsilon}(\log\frac{1}{\delta}+\abs{f_n}_{\mathcal{C}}+a_2) \geq m^{\epsilon,\delta}_\mathcal{C}(f_n,\mathcal{P})$, and thus
        \begin{equation}
            \frac{\frac{a_1}{\epsilon}(\log\frac{1}{\delta}+\abs{f_n}_{\mathcal{C}}+a_2)}{\frac{a_1}{\epsilon}(\log\frac{1}{\delta}+\abs{f_n}_{\mathcal{U}^c}+a_2)} \notin O(n^{1+\gamma}).
        \end{equation}
        
        Since $\epsilon$ and $\delta$ are fixed, the sequence of function $(f_n)$ satisfies
        \begin{equation}\label{eq:superlinear-circ-length}
            \frac{\abs{f_n}_\mathcal{C}}{\abs{f_n}_{\mathcal{U}^c}}\notin O(n^{1+\gamma}).
        \end{equation}
        
        We now restrict $n$ to the indices that form a sub-sequence of $(f_n)$ such that
        \begin{equation}
            \frac{\abs{f_n}_\mathcal{C}}{\abs{f_n}_{\mathcal{U}^c}}\in\Omega(n^{1+\gamma/2}).
        \end{equation}
        Let $N\subseteq \mathbb{N}$ denote the set of such indices. Notice that such a restriction remove any function $f$ such that $\abs{f}_{\mathcal{U}^c}=+\infty$ from the sequence.
        
        By definition of the big-$\Omega$ notation, there exists some $b>0$ such that for all $n$ sufficiently large
        \begin{equation}\label{eq:uncomp-gain}
            \abs{f_n}_\mathcal{C}\geq b\abs{f_n}_{\mathcal{U}^c}n^{1+\gamma/2}.
        \end{equation}
        
        We will now define a Boolean function computable in polynomial-time and, using Equation \ref{eq:uncomp-gain}, prove a superlinear circuit complexity for it.
        
        This Boolean function will be noted $I$ and is defined by
        \begin{equation}\label{eq:def-I}
            I(\scalarprod{x_1}{x_2}) = \mathcal{U}^{c}(x_1,x_2),
        \end{equation}
        where $\scalarprod{.}{.}:\mathcal{B}^*\times\mathcal{B}^*\rightarrow \mathcal{B}|\,\scalarprod{x_1}{x_2}=\underbrace{0\ldots0}_{\abs{x_2}}1x_1x_2$. The encoding $\scalarprod{.}{.}$ is bijective and invertible in polynomial-time. Thus $I$ is computable in polynomial-time since $\mathcal{U}^c$ is also computable in polynomial-time.
        
        We prove a lower-bound on the description-length of $I$ with the Boolean circuit interpreter on some inputs' sizes as we show next by contradiction. We denote $I|_{2n+1+\abs{f_n}_{\mathcal{U}^c}}$ the Boolean function $I$ restricted to $2n+1+\abs{f_n}_{\mathcal{U}^c}$ sized-inputs. We will show $\abs{I|_{2n+1+\abs{f_n}_{\mathcal{U}^c}}}_\mathcal{C}\geq \abs{f_n}_\mathcal{C}$.
        
        Suppose there exists some Boolean circuit that computes $I|_{2n+1+\abs{f_n}_{\mathcal{U}^c}}$, and that its description-length is strictly lower than $\abs{f_n}_\mathcal{C}$, i.e. $\abs{I|_{2n+1+\abs{f_n}_{\mathcal{U}^c}}}_\mathcal{C}< \abs{f_n}_\mathcal{C}$.
        
        For all $n$, let $p_n\in\mathcal{B}^{\abs{f}_{\mathcal{U}^c}}$ be the Boolean string such that $\mathcal{U}^c(p_n,.)=f_n$. By definition of $I$ (Equation \ref{eq:def-I}) the function $I(\scalarprod{p_n}{.})$ computes $f_n$, i.e. for all $x\in\mathcal{B}^n$ we have $I(\scalarprod{p_n}{x})=f_n(x)$.
        
        For the supposed circuit, we hardwire the $n+2$ to the $n+2+\abs{f_n}_{\mathcal{U}^c}$ inputs' variables to $p_n$. This force the circuit to compute $f_n$ as shown.
        
        When we hardwire some of the inputs' variables, the Boolean circuit size can only diminish. By Definition \ref{def:bool-circ} of the Boolean circuit interpreter, the description-length of a circuit is an increasing monotone function of its size. So, when we hardwire $p_n$ in the input, the circuit can only diminish in description-length. This implies that $\abs{f_n}_{\mathcal{C}}\leq \abs{I|_{2n+1+\abs{f_n}_{\mathcal{U}^c}}}_\mathcal{C} <  \abs{f_n}_\mathcal{C}$, this inequality is a contradiction. Consequently, we must have
         \begin{equation}
            \abs{I|_{2n+1+\abs{f_n}_{\mathcal{U}^c}}}_\mathcal{C}\geq \abs{f_n}_\mathcal{C}.
        \end{equation}
        
        With our lower-bounds on circuits' description-lengths of Equation \ref{eq:uncomp-gain}, it gives, for all $n$ sufficiently large,
        \begin{equation}
            \abs{I|_{2n+1+\abs{f_n}_{\mathcal{U}^c}}}_\mathcal{C}\geq bn^{1+\gamma/2}\abs{f_n}_{\mathcal{U}^c}.
        \end{equation}
        
        We cannot conclude the Theorem directly from this, the fact that $\abs{f_n}_{\mathcal{U}^c}$ can vary with $n$ complexifies the analysis. The rest of the proof address this issue by identifying an infinite subsequence of $(f_n)$ for which the evolution of $\abs{f_n}_{\mathcal{U}^c}$ has a tight characterization.
        
        We distribute the sequence of learning problems' functions, $f_n$, in different sets. For some precision parameter $\nu>0$, we define $i^{\max}=\ceil{\frac{1}{\nu}}d$, the sequence of indices $i=1,\ldots, i^{\max}$, the following sets
        \begin{equation}\label{eq:def-Si}
            S_i = \left\{ n\in N|\, \frac{i-1}{i^{\max}}n^{\frac{i-1}{i^{\max}}d} < \abs{f_n}_{\mathcal{U}^c} \leq \frac{i}{i^{\max}}n^{\frac{i}{i^{\max}}d} \right\},
        \end{equation}
        and $S_0$ containing the unique possible description-length of $0$ function.
        
        We have the upper-bound on the description-length of Equation \ref{eq:def-G-tilde-superlinear-true}, for all $n$, $\abs{f_n}_{\mathcal{U}^c}\leq n^d$. Thus, the functions in the infinite sequence $(f_n)$ are well partitioned in the defined sets. By the pigeon-hole principle there exists an index $i^*$ such that $S_{i^*}$ is of infinite size.
        
        We now restrict all $n$ to be in $S_{i^*}$. We prove that the description-length of the program $I$ is superlinear in the input-size with the interpreter for circuits for this sub-sequence of indices.
        
        We distingish three cases that cover all the possibilities for $i^*$:
        \begin{enumerate}
            \item Zero-length $i^*=0$.
            
            By Definition \ref{def:turing-mach} of Turing machines, a Turing machine has at least two states. By Definition \ref{def:enc-turing} of the encoder for Turing machines, a Turing machine with at least two states has at least a description-length of two bits. From these two facts and by Definition \ref{def:l-res-turing} of $\mathcal{U}^c$, a zero-length description interpreted by $\mathcal{U}^c$ outputs $\bot$. Thus, for all $n$ and any function $f\in H^n$, $\abs{f}_{\mathcal{U}^c}>0$. Consequently, the set $S_0$ is empty and this case is not possible.
            
            \item Sub-linear $i^*\leq \ceil{1/\nu}=i^{\max}/d\rightarrow \frac{i^*}{i^{\max}}d\leq 1$.
            
            We provide a lower-bound on the power linking the input-size to the circuits' description-lengths lower-bounds given in Equation \ref{eq:uncomp-gain}.
            
            We note that we have, $0<\abs{f_n}_{\mathcal{U}^c}\leq \alpha n$ for $\alpha=i^*/i^{\max}$ by Equation \ref{eq:def-Si}.
            
            Using properties of the logarithm, for all $\kappa_1,\kappa_2>0$, and all $n$ sufficiently large,
            \begin{equation}
            \begin{split}
                \log_{2n+1+\abs{f_n}_{\mathcal{U}^c}} bn^{1+\gamma/2}\abs{f_n}_{\mathcal{U}^c}&\geq \log_{(2+\alpha)n+1}n^{1+\gamma/2} + \log_{2n+1+\abs{f_n}_{\mathcal{U}^c}}b + \log_{2n+1+\abs{f_n}_{\mathcal{U}^c}}\abs{f_n}_{\mathcal{U}^c}\\
                &\geq \frac{\log_n n^{1+\gamma/2}}{\log_n ((2+\alpha)n+1)}-\kappa_1+0\\
                &\geq \frac{1+\gamma/2}{1+\kappa_2}-\kappa_1.
            \end{split}
            \end{equation}
            Since $\gamma>0$, there exists $\kappa_1,\kappa_2$ small enough such that this lower bound is strictly greater than one.
            
            \item Superlinear $i^*>\ceil{1/\nu}\rightarrow i^*-1\geq \ceil{1/\nu}=i^{\max}/d\rightarrow \frac{i^*-1}{i^{\max}}d\geq 1$.
            
            We pose $q=\frac{i^*}{i^{\max}}d$, $\nu'=\frac{1}{\ceil{1/\nu}}\leq\nu$, and $\alpha_1=\frac{i^*-1}{i^{\max}},\alpha_2=\frac{i^*}{i^{\max}}$.
            
            The following holds, by definition of $S_{i^*}$ in Equation \ref{eq:def-Si}, $\alpha_1n^{q-\nu'}\leq \abs{f_n}_{\mathcal{U}^c}\leq \alpha_2n^q$. Also, we have $q\geq1$ and $q-\nu'= \frac{i^*-1}{i^{\max}}d \geq 1$.
            
            In this case the power linking the input-size to the circuits' description-lengths lower-bounds is, for all $\kappa_1,\kappa_2,\kappa_3>0$ and all sufficiently large $n$,
            \begin{equation}
            \begin{split}
                \log_{2n+1+\abs{f_n}_{\mathcal{U}^c}}bn^{1+\gamma/2}\abs{f_n}_{\mathcal{U}^c}
                &\geq \log_{2n+1+\abs{f_n}_{\mathcal{U}^c}} \alpha_1 bn^{1+\gamma/2+q-\nu'}\\    
                &\geq \log_{(2+\alpha_2)n^q+1} n^{1+\gamma/2+q-\nu'} + \log_{2n+1+\abs{f_n}_{\mathcal{U}^c}} \alpha_1b\\
                &\geq \frac{\log_{n^q} n^{1+\gamma/2+q-\nu'}}{\log_{n^q} ((2+\alpha_2)n^q+1)} -\kappa_1\\
                &\geq \frac{1+\frac{1+\gamma/2-\nu'}{q}}{1+\log_{n^q} (2+\alpha_2) + \kappa_2} -\kappa_1\\
                &\geq \frac{1+\frac{1+\gamma/2}{d}}{1+\kappa_2+\kappa_3} -\frac{\nu}{1+\kappa_2+\kappa_3} -\kappa_1.
            \end{split}
            \end{equation}
            Our reasoning can be taken with arbitrarily small $\nu$ and $\kappa_1,\kappa_2,\kappa_3$, such that the lower bound can be made strictly greater than one.
        \end{enumerate}
        
        All the possible cases have been treated.
        
        The proved bound on the Boolean circuits' description-length extends to their sizes. More precisely, any superlinear lower-bound of the type $n^{1+\iota}$, for some $\iota>0$, on the description-length of the Boolean circuits with interpreter $\mathcal{C}$ implies, for all $n$ sufficiently large, a similar lower-bound, $n^{1+\tau}$, for some $0<\tau<\iota$, for the Boolean circuits' sizes by Definition \ref{def:bool-circ}.
        
        This finishes the proof.
    \end{proof}

\ifisfalse*
    
    \begin{proposition}{\cite{pavan2006some}.}\label{th:avr-hard-PH}
        For any $k_1,k_2\in\mathbb{N}^+$, there exists a language $L\in \textbf{P}^{\sum_2^p}$ such that for every circuit sequence $(C_1,\ldots,C_n,\ldots)$ whose circuits' sizes are at most $n^{k_1}$, the following holds
        \begin{equation}
            \Pr\limits_{x\in U(\mathcal{B}^n)}\left[L(x)=C_n(x) \right] \leq 1/2 + 1/n^{k_2},
        \end{equation}
        where $U(\mathcal{B}^n)$ denotes the uniform distribution on $\mathcal{B}^n$.
    \end{proposition}
    
    \begin{proof}[Proof of Theorem \ref{thm:if-false}]
        By contraposition, we suppose $\textbf{P}=\textbf{NP}$, then the polynomial hierarchy collapses, $\textbf{P}=\textbf{PH}$, and thus in particular $\textbf{P}=\textbf{P}^{\sum_2^p}$.
        
        Implying with Proposition \ref{th:avr-hard-PH} that for all $k_1$ and $k_2$ there exists a language in $\textbf{P}$ such that for all $n$ there does not exist a Boolean circuit of size smaller than $n^{k_1}$ which approximate the language with accuracy at least $\frac{1}{2}+\frac{1}{n^{k_2}}$ under the uniform distribution.
        
        Fix $\epsilon=\nicefrac{1}{4}$, any $\delta\in(0,1)$, any $d\in\mathbb{N}^+$, and $\gamma=1$.
        
        Take $k_1=2$ and $k_2=1$. For any $n\geq4$, to select a function of error rate at most $\nicefrac{1}{4}$ all circuits of size lower than $n^2$ must be eliminated. By Definition \ref{def:bool-circ} of the Boolean circuits' interpreter and Proposition \ref{prop:bc-expr}, the link between circuits' size and description-length is an increasing monotonic function, and thus, at least $n^2/b$ samples will be necessary to eliminate all these circuits for the learning algorithm $\mathit{MDL}^\mathcal{C}$, for some fixed constant $b>0$.
        
        Moreover, by Definition \ref{def:l-res-turing} of $\mathcal{U}^c$, since the language is in $\textbf{P}$ there exists some $c\in\mathbb{N}^+$ and some $s\in\mathcal{B}^*$ such that the language is computed by $\mathcal{U}^c(s,.)$. For any $d\in\mathbb{N}^+$ and for all sufficiently large $n$, we have $\abs{s}\leq n^d$.
        
        Thus we have some combination of parameters for which, for some constant $\abs{s}$,
        \begin{equation}
            G^d_{\mathcal{U}^c\rightarrow\mathcal{C}}(\epsilon,\delta,n)\geq \frac{n^2/b}{\frac{a_1}{\epsilon}(\log\frac{1}{\delta}+\abs{s}+a_2)}\in\Omega(n^2).
        \end{equation}
    \end{proof}

\underassumption*
    \begin{proposition}{\cite{arora2009computational}.}\label{prop:wrs-to-avrg}
        Let $S:\mathbb{N}\rightarrow\mathbb{N}$ and $f\in\textbf{E}$ such that Boolean circuits that decide $f|_n$ are at least of sizes $S(n)$ for every $n$. Then there exists a function $g\in\textbf{E}$ and a constant $b>0$ such that approximating $g|_n$ under the uniform distribution with accuracy at least $0.99$ requires Boolean circuits of sizes at least $S(n/b)/n^b$ for every sufficiently large $n$.
    \end{proposition}
    \begin{proof}[Proof of Theorem \ref{thm:underassumption}]
         By Definition \ref{def:G}, the sample efficiency gain is
        \begin{equation}
        \begin{aligned}
            G^d_{\mathcal{U}^c\rightarrow\mathcal{C}}(\epsilon,\delta,n)=& \sup_{f\in H^n,\mathcal{P}\in\Delta(\mathcal{B}^n)} & & \frac{m^{\epsilon,\delta}_\mathcal{C}(f,\mathcal{P})}{m^{\epsilon,\delta}_{\mathcal{U}^c}(f,\mathcal{P})} \\
            & \text{subject to} & & \abs{f}_{\mathcal{U}^c}\leq n^d.
        \end{aligned}
        \end{equation}
        
        Let's define a sequence of learning problems entailing our theorem.
    
        By Proposition \ref{prop:wrs-to-avrg} and the assumption there exists a language $g\in \textbf{E}$ such that $g|_n$ can only be approximated with accuracy $0.99$ under the uniform distribution by Boolean circuits of size at least in $\Omega(2^{\Omega(n)}/n^b)$ for some constant $b>0$.
        
        We define $g'$ to be the application of $g$ on the $a\log n$ first variables of the input, for some constant $a>0$. For every $n$, we define $D_n$ a probability distribution on $\mathcal{B}^n$, where $x_1$ denotes the first $a\log n$ variables of the input and $x_2$ the others, $U(\mathcal{B}^N)$ is the uniform distribution on $\mathcal{B}^N$,
        \begin{equation}
            D_n[x_1,x_2]=\begin{cases}
                U(\mathcal{B}^{a\log n})[x_1]&\text{if }x_2=\textbf{0}^{n-a\log n}\\
                0&\text{else}.
            \end{cases}
        \end{equation}
        Finally, for every $n$, we define the learning problem $(g'|_n,D_n)$.
        
        There exists a polynomial-time Turing machine that decides $g'$. Also the condition $\abs{g'|_n}_{\mathcal{U}^c}\leq n^d$ will always be satisfied for all $n$ large enough since the learning problem's function corresponds to a program/Turing-machine of fixed description-length.
        
        Using the guarantee of Proposition \ref{prop:pac-length}, $m^{\epsilon,\delta}_{\mathcal{U}^c}(g'|_n,\mathcal{P})\leq \frac{a_1}{\epsilon}(\log\frac{1}{\delta}+\abs{g'|_n}_{\mathcal{U}^c}+a_2)$, we deduce that $m^{\epsilon,\delta}_{\mathcal{U}^c}(f,\mathcal{P})$ is upper-bounded by a constant independent of $n$.
        
        For all $n$ sufficiently large, all the Boolean circuits that approximates $g'|_n$ with accuracy at least $0.99$ under $D_n$ have sizes at least in $\Tilde{\Omega}(n^{a\Omega(1)})$.
        
        Thus, by Proposition \ref{prop:bc-expr}, at least $\Omega(n^{a\Omega(1)})$ samples are necessary to solve the learning problem $(g'|_n,D_n)$ by $\mathit{MDL}^{\mathcal{C}}$ for all $n$ sufficiently large.
        
        Consequently, for any $\gamma>0$ there exists $a$ fixed and large enough, such that there exists $\epsilon\in(0,\nicefrac{1}{2}),\,\delta\in(0,1)$ and $c>0$ satisfying $G^d_{\mathcal{U}^c}(\epsilon,\delta,n)\in\Omega(n^{1+\gamma})$.
    \end{proof}

%% file: description-length-theorems.tex
\label{section:descr-length-criterion}
This section highlights the results focusing on description-length instead of PAC-learning gains. The proofs are similar to the proofs of the last section, and sometimes a reference to the proofs of the last section will be made.

The structure of the results' presentation is the same as for the study of PAC-learning gains. 

\begin{theorem}
    There exists a constant $q\in\mathbb{R}^+$ such that for all $n\in\mathbb{N}^+$, 
    \begin{equation}
        \sup_{f\in H^n} \frac{\abs{f}_\mathcal{U}}{\abs{f}_\mathcal{C}} \leq q.
    \end{equation}
\end{theorem}
\begin{proof}
    A technical detail is that, by Definition \ref{def:bool-circ}, for any $n$ and function $f\in H^n$, $\abs{f}_\mathcal{C}>0$. So, the denominator is always non-zero.
    
    The Proposition \ref{prop:KT}, affirms that there exists a constant $K$ such that for any $n$ and any function $f\in H^n$, the following holds $\abs{f}_\mathcal{U}\leq\abs{f}_\mathcal{C}+K$.
    
    This fact, with the fact that the denominator is never null, implies
    \begin{equation}
        \sup_{f\in H^n} \frac{\abs{f}_\mathcal{U}}{\abs{f}_\mathcal{C}}\leq \frac{\abs{f}_\mathcal{C}+K}{\abs{f}_\mathcal{C}}\leq K+1.
    \end{equation}
\end{proof}

\begin{theorem}
    We have
    \begin{equation}
        \sup_{f\in H^n} \frac{\abs{f}_\mathcal{C}}{\abs{f}_\mathcal{U}}\in \Omega(2^n).
    \end{equation}
\end{theorem}
\begin{proof}
    Let be the Turing machine $p^T$, computing function $p$, and let $p|_n$ denote function $p$ restricted to inputs of size $n$.
    
    \begin{itshape}
    For input $x$ of size $n$:
    \begin{enumerate}
        \item Compute the input size $n$.
        \item Enumerate all binary function in $H^n$ in some pre-defined fixed lexicographic order. For each function, $f$:
        \begin{enumerate}
            \item Compute the minimal description-length necessary to compute the function $f$ for a circuit according to interpreter $\mathcal{C}$:
            \begin{equation}
            \begin{aligned}
                &\min_{h\in \mathcal{B}^*} & & \abs{h} \\
                & \text{subject to} & & \mathcal{C}(h,y)=f(y)\quad\forall y\in\mathcal{B}^n.
            \end{aligned}
            \end{equation}
            \item If $\abs{h}\geq 2^n$ then return $f(x)$.
        \end{enumerate}
    \end{enumerate}
    \end{itshape}
    
    For any input-size $n$, there always exists a function that will satisfy the description-length condition that appears in step \textit{(b)}. We prove it by a counting argument, there are $2^{(2^n)}$ functions in $H^n$ and at most $2^{(2^n-1)}$ functions can be represented by a binary representation of length at most $2^n-1$.
    
    By construction and Definition \ref{def:univ-turing} of the universal Turing machine, the computed function is computable by the Turing machine $p^T$ and, thus, for any $n$, $\abs{p|_n}_\mathcal{U}\leq \alpha$ for some constant $\alpha$.
    
    Also by construction, for any $n$, the computed function is only computed by circuits of description-length at least $2^n$, $\abs{p|_n}_\mathcal{C}\geq 2^n$.
\end{proof}

\begin{theorem}
    For all $c\in\mathbb{N}^+$, we have
    \begin{equation}
        \sup_{f\in H^n} \frac{\abs{f}_{\mathcal{C}}}{\abs{f}_{\mathcal{U}^c}}\in O(n^c\log^2 n).
    \end{equation}
\end{theorem}
\begin{proof}
    The proof is the same as the proof of Theorem \ref{thm:circuit-power} pruned of the PAC-learning related terms.
\end{proof}

\begin{theorem}
    For all $1<c\in\mathbb{N}^+$, we have
    \begin{equation}
        \sup_{f\in H^n} \frac{\abs{f}_\mathcal{C}}{\abs{f}_{\mathcal{U}^c}}\in\Omega(n\log n).
    \end{equation}
\end{theorem}
\begin{proof}
    Consider the parity function $\oplus(x)=\sum_ix_i\mod 2$. For any $n$, let $\oplus|_n$ denote the function $\oplus$ restricted to inputs of size $n$.
    
    The parity function can be computed by a Turing machine in linear time, and thus, for some constant $\alpha$ and all $n$ sufficiently large, $\abs{\oplus|_n}_{\mathcal{U}^c}\leq \alpha$ according to Definition \ref{def:l-res-turing}.
    
    Moreover, for any $n$, if a circuit compute $\oplus|_n$ then it must have at least $n$ nodes to depend on all the input's variables. Then, by Definition \ref{def:bool-circ} linking the circuit's size to its description-length, $\abs{\oplus|_n}_\mathcal{C}\in \Omega(n\log n)$.
\end{proof}

\begin{theorem}\label{thm:dl-if-superlinear}
    If there exists $c,d\in\mathbb{N}^+$, and $\gamma>0$ such that
    \begin{equation}
    \begin{aligned}
        \left[\sup_{f\in H^n} \frac{\abs{f}_\mathcal{C}}{\abs{f}_{\mathcal{U}^c}}\quad \text{such that}\quad \abs{f}_{\mathcal{U}^c}\leq n^d\right] \notin O(n^{1+\gamma})
    \end{aligned}
    \end{equation}
    then there exists a language in $\textbf{P}$ not computable by a sequence of Boolean circuits whose sizes are in $O(n^{1+\tau})$ for some $\tau>0$.
\end{theorem}
\begin{proof}
    Take the proof of Theorem \ref{thm:if-true} beginning in Equation \ref{eq:superlinear-circ-length}.
\end{proof}

We note that the upper-bound $n^d$ on the description-length $\abs{f}_{\mathcal{U}^c}$ is not necessary for Theorem \ref{thm:dl-if-superlinear} to hold. Using Proposition \ref{prop:tur-lim-lim}, for any function relevant in the proof of the theorem, an upper-bound on the function description-length with $\mathcal{U}^c$ polynomial in the input-size holds. This fact can replace the bound in $n^d$ in the proof of Theorem \ref{thm:if-true}.

\begin{theorem}\label{thm:descr-length-if-false}
    If for all $c\in\mathbb{N}^+$, and all $\gamma>0$,
    \begin{equation}
        \sup_{f\in H^n} \frac{\abs{f}_\mathcal{C}}{\abs{f}_{\mathcal{U}^c}}\in O(n^{1+\gamma})
    \end{equation}
    then $\textbf{P}\neq\textbf{NP}$.
\end{theorem}
\begin{proposition}{\citet{kannan1982circuit}.}\label{prop:kannan}
    For any nonnegative integer $k$, there exists a language $L\in\sum^p_2$ such that $L$ is not computable by a sequence of circuits whose sizes are in $O(n^k)$, where $n$ is the input-size.
\end{proposition}
\begin{proof}[Proof of Theorem \ref{thm:descr-length-if-false}]
    By contraposition, suppose $\textbf{P}=\textbf{NP}$; then the polynomial hierarchy collapses and $\textbf{P}=\textbf{PH}=\sum^p_2$.
    
    By Proposition \ref{prop:kannan}, with $k=2$, $\textbf{P}$ has a language not computable by any circuit sequence whose sizes are in $O(n^2)$. Denote the function representing this language by $f$, and $f|_n$ its restriction to size $n$ inputs.
    
    Then, by Definition \ref{def:bool-circ} of $\mathcal{C}$, a sequence of circuits that computes the function does not have their description-length in $O(n^2)$. 
    
    Moreover, by Definition \ref{def:l-res-turing} of $\mathcal{U}^c$, there exists constants $c$ and $\alpha$ such that, for all $n$ sufficiently large, $\abs{f|_n}_{\mathcal{U}^c}\leq \alpha$ since there is a fixed Turing machine able to compute the function for all $n$.
\end{proof}

\begin{theorem}\label{thm:descr-length-underassumption}
    If there exist a language $g\in\textbf{E}$ such that $g|_n$ can only be computed by circuits of sizes at least $2^{\epsilon n}$ for some $\epsilon>0$; then for all $\gamma>0$ there exists $c\in\mathbb{N}^+$ such that
    \begin{equation}
        \sup_{f\in H^n} \frac{\abs{f}_\mathcal{C}}{\abs{f}_{\mathcal{U}^c}}\in\Omega(n^{1+\gamma}).
    \end{equation}
\end{theorem}
\begin{proof}
    For any $n$ pose $f_n$ to be $g$ applied to the first $a\log n$ variables of the input, for some $a>0$. By construction $\abs{f_n}_\mathcal{C}\in\Omega(n^{a\Omega(1)})$, and since there exists a polynomial-time Turing machine deciding $(f_n)$, $\abs{f_n}_{\mathcal{U}^c}\in O(1)$, for $c$ large enough.
    
    Fix $a$ large enough to complete the proof.
\end{proof}

%% file: appendix.tex
\section{Interpreters}\label{section:interpreters}
    \subsection{Universal Turing Machine}
        We restrict our Turing machines to binary-valued outputs in the whole work. 
        
        The definitions of this sub-section are given with the number of working-tape let as a variable in some cases. The results of this research are correct for any value of this parameter.
        
        \begin{definition}{\emph{Turing machines}.}\label{def:turing-mach}
            We first define one-tape Turing machines before generalizing the definition to multi-tape Turing machines.
            
            A binary-valued one-tape Turing machine is defined by
            \begin{itemize}
                \item $Q$ a finite set of states;
                \item $A=\{0,1,b\}$ the Turing machine's alphabet;
                \item $q_0\in Q$ the initial state in which the Turing machine begins;
                \item $\{\mathit{accept},\mathit{reject}\}=F\subset Q$ the set of the two final states which determine the output of the Turing machine: if the Turing machine stops in the $\mathit{accept}$ state then the output is $1$, else if it stops in the final state $\mathit{reject}$ then the output is $0$, else it is $\bot$;
                \item a mapping from $((Q\backslash F)\times A)$ to $(Q\times S)$ defining the transition of the Turing machine, where $S=A\cup\{L,R\}$ which correspond to either writing the element of $A$ to the current head place or move the head to the $L$:left or $R$:right.
            \end{itemize}
            The partial computable function implemented by the Turing machine is defined by setting the binary input on the unique tape in contiguous cells (the $b$ symbol being assigned to the other cells); positioning the head on the first cell; set the Turing machine in the $q_0$ state; then to apply recursively the mapping of the Turing that determines the change of states, writing on the tape, and head movements; finally the output is obtained either when a final state in $F$ is obtained, or else by $\bot$.
            
            For $k_1,k_2\in\mathbb{N}^+$, we define $k_1$-input-tape, $k_2$-working-tape Turing machines. These machines have $k_1$ binary-inputs that are placed on the $k_1$ inputs' tapes. There is one head by tape and these tapes are read-only. There are also $k_2$-working-tapes with one head by tape, they are blank at the start, and are read and write. 
            
            These machines are determined in a similar way to one-tape Turing machine: the elements $Q$, $A$, $q_0$, $F$, and $S$ are defined in the same way. The mapping is adapted, the mapping goes from $((Q\backslash F)\times A^{k_1+k_2})$ to $(Q\times \{L,R\}^{k_1}\times S^{k_2})$, with the natural interpretation.
            
            The partial computable function implemented by such Turing machines follows from a natural generalization of one-tape Turing machines.
            
        \end{definition}
        
        \begin{definition}{\emph{Turing machines encoding $E(.)$}, \cite{li2019introduction}.}\label{def:enc-turing}
            Following Definition \ref{def:turing-mach}, any Turing machine, $T$, can be fully described by a set of states $Q$, the initial state $q_0$, and a mapping from $((Q\backslash F)\times A)$ to $(Q\times S)$.
            
            The mapping and $T$ can be described by a list of quadruples $[(p_i,t_i,q_i,s_i)]_{i=1}^r$ where $r$ is the number of rules and for all $i$, $p_i,q_i\in Q,\,t_i\in A,\,s_i\in S$. Each element can be identified with $s=\ceil{\log(\abs{Q}+5)}$ bits. Be $e(.):Q\cup S\rightarrow\mathcal{B}^s$ this encoding. By convention this encoding will satisfy the following constraint, the states $q_0$, $\mathit{accept}$, and $\mathit{reject}$ will be encoded to predefined arbitrary values (the three first elements in the Boolean lexicographic order of the output for example).
            
            We define the encoding of $T$ to be
            \begin{equation}
                E(T) = \underbrace{0\ldots0}_{s}1\underbrace{0\ldots0}_{r}1[e(p_i)e(t_i)e(s_i)e(q_i)]^r_{i=1}.
            \end{equation}
            
            This encoding completely defines the Turing machine $T$.
            
            The encoding is prefix-free: no encoding is the prefix of another.
            
            To define an encoding for multiple-tape Turing machines, generalize the encoding to their mappings defined in Definition \ref{def:turing-mach}.
        \end{definition}
        
        \begin{definition}{\emph{Universal Turing machine $\mathcal{U}$}.}\label{def:univ-turing}
            For any $k\in\mathbb{N}^+$, the interpreter $\mathcal{U}:\mathcal{B}^*\times \mathcal{B}^*\rightarrow \mathcal{B}\cup\{\bot\}$ computes $T(u,x)$ on input $([E(T),u],x)$, where $E$ follows Definition \ref{def:enc-turing} for two-input-tape and $k-$working-tape Turing machines. If the input has not a form that encodes a Turing machine then $\bot$ is the output.
            
            Note that the decomposition of the first argument in $E(T)$ and $u$ is well defined since the encoding $E$ is prefix-free.
        \end{definition}
        
        \begin{definition}{\emph{Polynomial-time universal Turing machines $\mathcal{U}^c$}.}\label{def:l-res-turing}
    
            For any $k,c\in\mathbb{N}^+$, we define an interpreter $\mathcal{U}^c:\mathcal{B}^*\times \mathcal{B}^*\rightarrow \mathcal{B}\cup\{\bot\}$. It is a $2$-input-tape, $3$-working-tape Turing machine.
            
            On input $([E(T),u],x\in\mathcal{B}^*)$, for $T$ a $2$-input-tape $k$-working tape Turing machine, $E$ the encoding in Definition \ref{def:enc-turing}, and $u\in\mathcal{B}^*$; the following operations are performed:
            \begin{enumerate}
                \item On the third working tape, the interpreter computes the input-size, $n$, of the second input, $x$.
                \item Still on the third work-tape, it computes $n^c$.
                \item The interpreter computes in at most $n^c$ steps that the form of the first input corresponds to the encoding of a Turing machine. If it does not correspond to a Turing machine or if the number of steps limit is reached, it outputs $\bot$.
                \item Then, it computes a simulation of the behavior of the Turing machine $T$ on input $(u,x)$ with the two first work-tapes using the construction of \cite{hennie1966two}, whose result is given in Proposition \ref{prop:efficient-turing}. Simultaneously, the interpreter computes the number of steps dedicated to the simulation on the third work-tape. (Note that it is well the number of steps dedicated to the simulation and not the number of simulated steps that are counted.)
                \item In the computed simulation if a final state in $F$ is reached then enters this state for the universal Turing machine. If the limit of computation for the simulation, $n^c$, is attained without entering a final state of $F$ in the computed simulation then enter the state $\mathit{reject}$.
            \end{enumerate}
            
            For all these operations, the total number of steps for the first input, $[E(T),u]$, fixed can be made in $\beta n^c$, for some fixed $\beta>0$.
            
            Using Proposition \ref{prop:efficient-turing}, for any $\delta>0$, any Turing machine with computational complexity in $O(n^{c-\gamma})$ can be simulated, for some $h\in\mathcal{B}^*$ and all $n$ sufficiently large, by $\mathcal{U}^c(h,.)$.
        \end{definition}
        
        \begin{proposition}{Efficient universal Turing machines, \cite{hennie1966two}, \cite{arora2009computational}.}\label{prop:efficient-turing}
            There exists a universal Turing machine which, for any Turing machine $T$, on inputs $E(T)$ and $x$ computes $T(x)$.
            
            Furthermore, for some $\alpha>0$, if the Turing machine $T$ on input $x$ stops in $t$ steps then the universal Turing machine stops in $\alpha t\log t$.
        \end{proposition}
        
        \begin{proposition}{\emph{Hardwiring}.}\label{prop:hardwire}
            For any $2$-input-tape Turing machine, $T$, and input $h\in\mathcal{B}^*$ there exists a Turing machine $T^h$ such that $T^h$ compute the function $T(h,.)$.
            
            Moreover, $T^h$ has $\rho \abs{h}$ rules, for some $\rho$ independent of $h$; and $T^h$ computes the function $T(h,.)$ in the same number of steps.
            
        \end{proposition}
        \begin{proof}
            Let $Q$ be the states and $R$ be the set of rules of $T$ that defines its mapping, construct the new states $Q\times \{1,\ldots,\abs{h}\}$ and the new rules $R\times\{1,\ldots,\abs{h}\}$. The new rules are made such that any operation of the Turing machine $T$ on the first input-tape is translated into an equivalent change in the state of the Turing machine $T^h$. Allowing the simulation of the tape's head corresponding to input $h$ in the states of $T^h$.
            
            In our construction there are thus $R\cdot\abs{h}$ rules in $T^h$, fix $\rho=R$ in the theorem statement.
        \end{proof}
        
        \begin{proposition}\label{prop:tur-lim-lim}
            There exists a constant $\beta>0$ such that for any function $f$ in $H_n$, if $\abs{f}_{\mathcal{U}^c}<+\infty$ then $\abs{f}_{\mathcal{U}^c}\leq \beta n^c$.
        \end{proposition}
        \begin{proof}
            Any solution $h$ of length larger than $\beta n^c$ can be made smaller by cropping all binary symbols after index $\beta n^c$ on the tape, since, by Definition \ref{def:l-res-turing}, $\mathcal{U}^c$ cannot read them in $\beta n^c$ steps.
        \end{proof}

    \subsection{Boolean circuit}
        \begin{definition}{\emph{Boolean circuit}, \cite{arora2009computational}.}\label{def:bool-circ-basic}
            A Boolean circuit $C$ is a directed acyclic graph with $n\in\mathbb{N}$ potential sources and one sink. The source vertices have an associated input variable whose index is between $1$ and $n$. The non-source vertices are called gates and have an associated logical operation OR, AND, or NOT ($\land$, $\lor$ or $\lnot$) called label.
            
            The OR and AND vertices have two input edges, the NOT vertices have one input edge.
            
            The number of vertices will be denoted $\abs{C}$ and called the size of the circuit.
            
            The output of the circuit on an input $x\in\mathcal{B}^n$ is the value associated to the sink vertice applying recursively the following assignment for each vertex $v$: if $v$ is a source corresponding to input variable $i$ then its value is $x_i$; else $v$ is a gate, apply the logical operator corresponding to its label on the input values (values from its parent vertices).
        \end{definition}
        
        \begin{definition}{\emph{Boolean circuit interpreter}, $\mathcal{C}$.}\label{def:bool-circ}
            We define $\mathcal{C}:\mathcal{B}^*\times \mathcal{B}^*\rightarrow \mathcal{B}\cup\{\bot\}$ to compute circuit $C(x)$ on input $(h,x)$ with $h$ of the form
            \begin{equation}
                \underbrace{0\ldots 0}_{\ceil{\log_2n}}1[\text{binary description of }n]\underbrace{0\ldots 0}_{\abs{C}}1[\text{label and inputs' vertices of vertice }i]_{i=1}^{\abs{C}},
            \end{equation}
            where label and inputs' vertices of any vertice correspond to $2+\max\{2\ceil{\log_2 \abs{C}},\ceil{\log_2 n}\}$ bits; $2$ bits to denote its logical label, and $2\ceil{\log_2 \abs{C}}$ bits for the input vertices.
            
            It outputs $\bot$ if $h$ has not an acceptable form.
            
            The description-length of $h$ is thus of $2\ceil{\log_2(n)}+2+\abs{C}(3+\max\{2\ceil{\log_2\abs{C}}),\ceil{\log_2 n}\}$ bits.
            
            Which can be bounded by $9\abs{C}\log\abs{C}$ bits for circuits of sizes $\abs{C}\geq n\geq 3$.
        \end{definition}
        
        \begin{proposition}{\cite{frandsen2005reviewing}.}\label{prop:circ-suff-size}
        
            Boolean circuits of size at most
            \begin{equation}
                \frac{2^n}{n}(1+3\frac{\log(n)}{n}+O(\frac{1}{n}))
            \end{equation}
            compute all functions in $H^n$.
        \end{proposition}
    
    \subsection{Artificial Neural Network}
        Similarly to Boolean circuits, we define an interpreter for Artificial Neural Networks (ANNs) and provide propositions linking the two definitions in terms of description-length.
        \begin{definition}{\emph{Floating-point operators}.}
            For any $d\in\mathbb{N}^+$, a floating-point operator is:
            \begin{itemize}
                \item a $0$-ary/constant operator, which is a float-number ---an element in $\mathcal{B}^d$;
                \item an unary operator from $\mathcal{B}^d$ to $\mathcal{B}^d$, such as negation $-(.)$, inverse $1/(.)$, or the exponential $\exp(.)$;
                \item a binary operator from $\mathcal{B}^d\times\mathcal{B}^d$ to $\mathcal{B}^d$, such as addition $(.+.)$, or product $(.\cdot.)$.
            \end{itemize}
        \end{definition}
        \begin{definition}{\emph{Artificial Neural Network}.}\label{def:ANN}
            Given $d\in\mathbb{N}^+$ and a fixed pre-defined finite set of floating-point operators $\mathcal{O}$ on $\mathcal{B}^d$ with at least two $0$-ary operators identified as $0^F$ and $1^F$ (thus both in $\mathcal{B}^d$).
            
            For any $n\in\mathbb{N}^+$, an ANN is a directed acyclic graph with one sink and at most $n$ input-variable sources. The sources vertices have each an associated input variable whose index is between $1$ and $n$. The non-source vertices have an associated operator taken from the set $\mathcal{O}$. If the operator is $0$-ary then they have no parent vertice, if it is unary then they have one parent vertice, else if the operator is binary then they have two parent vertices.
            
            We denote $\abs{A}$ the number of vertices of ANN $A$.
            
            To compute the output of the ANN on input $x\in\mathcal{B}^n$, the following operations are performed:
            \begin{enumerate}
                \item for all the input-variable sources: the floating-point operator $0^F\in\mathcal{B}^d$ or $1^F\in\mathcal{B}^d$ is assigned accordingly to the associated input variable value in $\{0,1\}=\mathcal{B}$.
                \item for all the other vertices assign the value in $\mathcal{B}^d$ corresponding to the associated floating-point operator in $\mathcal{O}$ and the values assigned to the potential parents.
                \item when the output sink vertex has been assigned a value: if it is $0^F$ then output $0\in\mathcal{B}$, if it is $1^F$ then output $1\in\mathcal{B}$, else output $\bot$.
            \end{enumerate}
        \end{definition}
        
        \begin{definition}{\emph{Artificial Neural Network interpreter}, $\mathcal{ANN}^\mathcal{O}$.}\label{def:ANN-int}
            We define $\mathcal{ANN}^\mathcal{O}:\mathcal{B}^*\times\mathcal{B}^*\rightarrow\mathcal{B}\cup\{\bot\}$ the interpreter for ANN. On input $(g,x)$ it computes the result of applying ANN $A$ on $x$, $A(x)$, where $g$ is of the form
            \begin{equation}
                \underbrace{0\ldots 0}_{\ceil{\log_2n}}1[\text{binary description of }n]\underbrace{0\ldots 0}_{\abs{A}}1[\text{input variable/operator and inputs' vertices of vertice }i]_{i=1}^{\abs{A}},
            \end{equation}
            the operator in $\mathcal{O}$ and the parent(s) or the input variable will be described in $\ceil{\log_2 (\abs{\mathcal{O}}+1)}+\max\{2\ceil{\log_2\abs{\mathcal{O}})},\ceil{\log_2 n}\}$ bits for each vertex.
            
            If $g$ does not have a correct form then $\bot$ is returned.
            
            The length of $g$ encoding an ANN $A$ is $2\ceil{\log_2 n}+2+\abs{A}(1+\ceil{\log_2(\abs{\mathcal{O}}+1)}+\max\{2\ceil{\log_2 \abs{A}}, \ceil{\log_2 n}\})$
        \end{definition}
        
        \begin{proposition}\label{prop:ANN-bool-circuit-descr-length}
            For any fixed set of floats' operators $\mathcal{O}$ for which the AND, OR, and NOT Boolean functions can each be computed by an ANN using the operators in $\mathcal{O}$, there exists constant $\alpha,\beta>0$ such that the following holds.
            
            For any $n\in\mathbb{N}^+$ and any $h\in\mathcal{B}^*$ there exists $g\in\mathcal{B}^*$ such that for all $x\in\mathcal{B}^n$ 
            \begin{equation}
                \mathcal{C}(h,x)=\mathcal{ANN}^\mathcal{O}(g,x)
            \end{equation}
            and
            \begin{equation}
                \abs{g}\leq\alpha\abs{h}.
            \end{equation}
            
            Conversely, for any $n\in\mathbb{N}^+$ and any $g\in\mathcal{B}^*$ there exists $h\in\mathcal{B}^*$ such that for all $x\in\mathcal{B}^n$ 
            \begin{equation}
                \mathcal{ANN}^\mathcal{O}(g,x)=\mathcal{C}(h,x)
            \end{equation}
            and
            \begin{equation}
                \abs{h}\leq\beta\abs{g}.
            \end{equation}
        \end{proposition}
        \begin{proof}
            For the first part, if $h$ has not a correct form to represent a circuit then $\mathcal{C}(h,.)$ always output $\bot$ which can also be done by a $g$ that do not represent an ANN.
            
            Otherwise, there is a circuit, $C$, that corresponds to $\mathcal{C}(h,.)$, the logical operation of each node of this circuit can be simulated by a part of an ANN of fixed maximal size. A combination of these parts gives an ANN, $A$, which computes the same function as the Boolean circuit and whose size is at most a fixed multiple of the size of the circuit.
            
            Let's pose $\gamma>0$ the factor of the sizes, $\abs{A}\leq\gamma\abs{C}$. The description-lengths have the following relationship
            \begin{multline}
                2\ceil{\log_2 n}+2+\abs{A}(1+\ceil{\log_2(\abs{\mathcal{O}}+1)}+\max\{2\ceil{\log_2 \abs{A}}, \ceil{\log_2 n}\})\\
                \leq 2\ceil{\log_2 n}+2+\gamma\abs{C}(1+\ceil{\log_2(\abs{\mathcal{O}}+1)}+\max\{2\ceil{\log_2\abs{C}}+2\ceil{\log_2\gamma}, \ceil{\log_2 n}\})\\
                \leq \alpha\left(2\ceil{\log_2(n)}+2+\abs{C}(3+\max\{2\ceil{\log_2\abs{C}}),\ceil{\log_2 n}\}\right),
            \end{multline}
            with $\alpha=\max\{1,\gamma,(1+\ceil{\log_2(\abs{\mathcal{O}}+1)}+2\ceil{\log_2\gamma})/3\}$.
            
            A similar argument holds for the second part by noting that any floating-point operator can be computed by a finite size Boolean circuit by Proposition \ref{prop:circ-suff-size}.
        \end{proof}

\section{VC-dimension analysis and tightness of Proposition \ref{prop:pac-length}}\label{section:tight-prop-and-vc-dim}
    \begin{definition}{\emph{Shatter}.}
        Let be a set $C$.
        
        A set $A\subset 2^C$ shatters a set $B\subset C$ iff
        \begin{equation}
            \{a\cap B|\,a\in A\}=2^{\abs{B}}.
        \end{equation}
    \end{definition}
    
    \begin{definition}{\emph{VC-Dimension}, \cite{vapnik2015uniform}.}
        The VC-dimension of a set $F\subset H^n$, $\mathit{VC}(F)$, is the size of the largest set $B\in\mathcal{B}^n$ such that $F$ shatters $B$, where each binary-valued function in $F$ is translated as a set in $\mathcal{B}^n$.
    \end{definition}

    \begin{proposition}{VC-Dimension of Turing machines.}\label{prop:vc-turing}
        There exists a constant $a>0$ such that the following holds.
        
        For any $n$ and $D$ in $\mathbb{N}^+$, with $D$ larger than some constant, we define the sets of functions $F^D=\{\mathcal{U}(h,.)\in H^n|\,h\in\mathcal{B}^*\,\abs{h}\leq D\}$.
        
        These sets satisfy $\mathit{VC}(F^D)\geq a D$.
    \end{proposition}
    \begin{proof}
        Consider the two inputs Turing machine, $T(u,x)$, that output $u_x$ if $x\leq \abs{u}$ and $0$ else, where the binary input string $x$ is understood as the binary representation of a number.
        
        From this Turing machine create for each $D$ the set of functions $\{\mathcal{U}([E(T),u],.)|\,u\in\mathcal{B}^{\floor{D/2}}\}$, this set shatters the set of the $aD$ strings corresponding to the $aD$ first natural numbers in binary representation, for some $a>0$. Moreover, if $D$ is bigger than $2\cdot\abs{E(T)}$, the construction satisfies the upper-bound of $D$ on the description-lengths.
    \end{proof}
    
    \begin{proposition}{VC-Dimension of Boolean circuits.}\label{prop:vc-circ}
        There exists constants $a,b,N$ such that the following holds for all $n\geq N$.
        
        Be the sets of functions $F^D=\{\mathcal{C}(h,.)\in H^n|\,h\in\mathcal{B}^*\,\abs{h}\leq D\}$ for some constant $D\in\mathbb{N}$, with $D\geq bn^{1.01}$.
        
        Then $\mathit{VC}(F^D)\geq aD$.
    \end{proposition}
    \begin{proof}
        We pose $b=b_1b_2$, for two positive constants $b_1$ and $b_2$.
    
        We have $1.01\log n\leq \log D - \log b$. Select $q=\floor{\log D-\log b_1}$ an integer and $b_2\geq e$ for $1.01\log n\leq q$ to hold.
        
        Consider the inputs $\{[z\underbrace{0\ldots 0}_{n-q}]\in\mathcal{B}^n|\,z\in\mathcal{B}^q\}$. This set is shattered by a set of circuits of size lower than $2\frac{2^q}{q}$ for $n$ (and thus $q$) sufficiently large by Proposition \ref{prop:circ-suff-size}.
        
        For all $n$ sufficiently large, we have $2\frac{2^q}{q}\geq 2\frac{n^{1.01}}{1.01\log n}\geq n$.
        
        Thus, by Definition \ref{def:bool-circ} these circuits can be expressed with
        \begin{equation}\label{eq:bits-descr-vc-dim-proof}
            18\frac{2^q}{q}\log \frac{2^q}{q} = \frac{18}{\log_2 e}2^q(1-\log^{-1}_2(e)\frac{\log q}{q}+\frac{\log 2}{\log_2(e)}\frac{1}{q})
        \end{equation}
        bits.
        
        There exists a constant $b_1$ sufficiently large such that, for all $n$ sufficiently large (forcing $q$ to be sufficiently large), $q\leq \log D -\log b_1\equiv 2^{\log b_1}2^q\leq D$ implies that the description-length of the circuits, as bounded by Equation \ref{eq:bits-descr-vc-dim-proof}, is smaller than $D$.
        
        Finally, take $a=2^{-\log b_1-1}$, we have $\mathit{VC}(F^D)\geq 2^q\geq aD$ since $q\geq \log D-\log b_1-1$.
    \end{proof}
    
    \begin{proposition}{PAC-learning lower-bound, \cite{shalev2014understanding}.}\label{prop:pac-lower-vc}
        There exists a constant $\alpha$ such that the following holds for any $\epsilon\in(0,\nicefrac{1}{2}),\delta\in(0,1)$, and any $n\in\mathbb{N}^+$.
        
        Consider the set of learning problems $(f,\mathcal{P})$ where $f\in F\subset H^n$ and $\mathcal{P}\in\Delta(\mathcal{B}^n)$ a probability measure.
        
        For any learning algorithm, there exists a learning problem in the set such that a $m$-sample dataset of the considered learning problem with
        \begin{equation}
            m\geq\frac{\alpha}{\epsilon}\left[\mathit{VC}(F)+\log(\frac{1}{\delta})\right]
        \end{equation}
        is necessary to get an $(\epsilon,\delta)$-PAC-learning performance.
    \end{proposition}

    \begin{proposition}{Tightness of Proposition \ref{prop:pac-length}.}\label{prop:tight-pac-length}
        There exists a constant $\alpha$ such that for any $\epsilon\in(0,\nicefrac{1}{2}),\delta\in(0,1),n\in\mathbb{N}^+$ and any interpreter $\varphi$, associated learning algorithm $\mathit{MDL}^\varphi$, and bound $D\in\mathbb{N}^+$ on the description-length of an underlying function to learn; there exists a learning problem such that a $m$-sample learning dataset with
        \begin{equation}
            m\geq\frac{\alpha}{\epsilon}\left[\mathit{VC}(\{f\in H^n|\,\abs{f}_\varphi\leq D\})+\log(\frac{1}{\delta})\right]
        \end{equation}
        is necessary for $\mathit{MDL}^\varphi$ to have an $(\epsilon,\delta)$-PAC-learning performance on the learning problem.
        
    \end{proposition}
    \begin{proof}
        The statement is a direct consequence of Proposition \ref{prop:pac-lower-vc}.
    \end{proof}

\section{Kolmogorov complexity}
    The following proposition comes from Kolmogorov complexity theory, it is adapted to the context and notation of this paper. See \cite{li2019introduction} for a reference on the subject. The origins of the theorem can be found in \cite{solomonoff1960preliminary,solomonoff1962inductive,solomonoff1964formal1,solomonoff1964formal2} and \cite{kolmogorov1965three}.
    \begin{proposition}{Invariance Theorem.}\label{prop:KT}
        For all interpreters $\varphi$ there exists a constant $K$ such that for all $n\in\mathbb{N}^+$ and all functions $f\in H^n$, the following holds $\abs{f}_\mathcal{U}\leq \abs{f}_\varphi + K$.
    \end{proposition}
    \begin{proof}
        Be $f^\varphi$ the string of length $\abs{f}_\varphi$ such that $\varphi(f^\varphi,.)=f(.)$.
        
        Take $E(\varphi^T)$ the encoding of the Turing machine corresponding to the interpreter $\varphi$, the encoding follows Definition \ref{def:enc-turing}.
        
        The function $\mathcal{U}([E(\varphi^T),f^\varphi],.)$ is equal to $f$ by Definition \ref{def:univ-turing}, and the string $[E(\varphi^T),f^\varphi]$ has length $\abs{f}_\varphi+K$ where $K=\abs{E(\varphi^T)}$ is independent of $f$. 
    \end{proof}

\section{Technical Propositions}
    \subsection{Number of necessary samples for Boolean Circuits}
        \begin{definition}\label{def:bin-dec-tree}
            For any $n$, a \emph{binary decision tree} is determined by
            \begin{itemize}
                \item a tree where all non-leaf nodes have exactly 3 neighbors: a first child, a second child, and one parent with the exception of one node which has no parent and is called the root;
                \item to each non-leaf node is associated an input-Boolean-variable;
                \item to each leaf is associated a Boolean value.
            \end{itemize}
            
            The binary-valued function computed by the binary decision tree on an input $x\in\mathcal{B}^n$ is the result of the following computation:
            \begin{itshape}
            \begin{enumerate}
                \item The current node is set to be the root.
                \item If the current node is not a leaf, then if the associated input-Boolean-variable in $x$ is $1$ then set the current node as the first child, else set the second child as the current node. Else continue to the next step.
                \item The current node is thus a leaf, return as output the Boolean value associated to the current node.
            \end{enumerate}
            \end{itshape}
        \end{definition}
        \begin{proposition}\label{prop:tree-power}
            Given $m$ samples of a binary function there exists a binary decison tree with at most $2m-1$ nodes consistent with the samples.
        \end{proposition}
        \begin{proof}
            In an optimal-size binary decision tree, there are at most $m$ leaves, one for each sample. By induction, we can show that in any binary decision tree there is at most the number of leaves minus one internal node.
        \end{proof}
        
        \begin{proposition}\label{prop:tree-circuit}\label{prop:circ-sim-tree}
            There exists $a>0$ such that if there exists a binary decision tree of size $S$ computing a boolean function then there exists a Boolean circuit of size at most $aS$ computing that function.
        \end{proposition}
         \begin{proof}
            
            Create $S-1$ nodes, one for each node of the tree except the root. Each node has its value defined by the AND of the parent node and of either the input-Boolean-variable associated with the parent node if it is the first child or of its negation if the second child. To achieve this each time an input-variable is needed, create a node associated with the input-variable in the circuit.
            In the case of the second child, one supplementary node associated with the unary logical operator NOT is used to compute the negation of the corresponding input-variable.
            
            For an input $x$, the values obtained at the nodes that correspond to the leaves ---let's denote them $b^t$ for leaf $t$-- are all $0$ except for the leaf at which the procedure described in Definition \ref{def:bin-dec-tree} terminates.
            
            A network of logical operators can aggregate the output associated with the leaves ---let's denote them $o^t$ for leaf $t$--- and output the answer associated with the only leaf to which $1$ has been associated.
            
            To do so consider the following two Boolean variables of $x_1,x_2\in\mathcal{B}^n$, with $x_1^0,x_2^0=0$, and with the following update for leaf number $t\in\mathbb{N}$ with associated output Boolean value $o_t$ and Boolean node value at runtime $b_t$,
            \begin{equation}
                \begin{bmatrix}
                    x_1^{t+1}\\
                    x_2^{t+1}
                \end{bmatrix}
                =
                \begin{bmatrix}
                    \text{if $x_2^t$ then }x_1^t; \text{else}\;o^t\\
                    \text{if $x_2^t$ then }x_2^t; \text{else}\;b^t
                \end{bmatrix}.
            \end{equation}
            
            This system iterated for all the leaves computes the output of the binary decision tree in $x_1$, and this iteration can be computed with a number of nodes that scale linearly with the number of leaves. 
            
            The final circuit size is in $O(S)$.
         \end{proof}
         
         \begin{proposition} \label{prop:bc-expr}
            There exists a constant $b>0$ such that given $m$ samples of a binary function there exists a Boolean circuit of size at most $bm$ consistent with the samples.
         \end{proposition}
         \begin{proof}
            Merge the last two results, Propositions \ref{prop:tree-power} and \ref{prop:circ-sim-tree}, to first produce a binary decision tree, then to convert it into a Boolean circuit of suitable size.
         \end{proof}
     
     \subsection{A combinatoral Proposition}
        An useful Definition and Proposition, it comes from \cite{lint1999introduction}.
        \begin{definition}\label{def:bin-ent}
            The \emph{binary entropy function} $H$ is defined by
            \begin{equation}
                H(x) = \begin{cases}
                    0&\text{if }x=0;\\
                    -x\log_2 x -(1-x)\log_2 (1-x),\,0<x\leq\frac{1}{2}&\text{else if }0<x\leq\nicefrac{1}{2}.
                \end{cases}
            \end{equation}
        \end{definition}
        
        \begin{proposition}\label{prop:approx_count}
            Let $0\leq\epsilon\leq \frac{1}{2}$ and $M\in\mathbb{N}^+$, we have
            \begin{equation}
                \sum_{0\leq i \leq \floor{\epsilon M}} \binom{M}{i} \leq 2^{M H(\epsilon)}.
            \end{equation}
        \end{proposition}